\documentclass{article}

\usepackage{amsthm}
\usepackage{wrapfig}
\usepackage[utf8]{inputenc} 
\usepackage[T1]{fontenc}    
\usepackage{url}            
\usepackage{booktabs}       
\usepackage{amsfonts}       
\usepackage{nicefrac}       
\usepackage{microtype}      
\usepackage{xcolor,colortbl}         
\usepackage{wrapfig}
\usepackage{caption}
\usepackage{enumitem}
\usepackage{tabularx}

\usepackage{mylatexstyle}
\usepackage[final]{neurips_2024}
\usepackage[compact]{titlesec}
\usepackage{setspace}

\usepackage{todonotes}

\allowdisplaybreaks

\definecolor{LightCyan}{rgb}{0.8, 0.9, 1}

\newcommand{\good}{\textit{good}}

\newcommand{\bad}{\textit{bad}}

\newcommand{\la}{\langle}
\newcommand{\ra}{\rangle}

\newtheorem{condition}[theorem]{Condition}

\def \poly {\mathrm{poly}}

\renewcommand{\zeta}{\overline{\rho}}

\renewcommand{\omega}{\underline{\rho}}

\def \poly {\mathrm{poly}}
\def \CC {}

\DeclareMathOperator{\polylog}{\rm polylog}

\title{Matching the Statistical Query Lower Bound for $k$-Sparse Parity Problems
with Sign Stochastic Gradient Descent}

\author{
 Yiwen Kou\thanks{Equal contribution} \\ Department of Computer Science \\
 University of California, Los Angeles \\
Los Angeles, CA 90095, USA \\
\texttt{evankou@cs.ucla.edu}
    \And
     Zixiang Chen\footnotemark[1]  \\
    Department of Computer Science \\
 University of California, Los Angeles \\
Los Angeles, CA 90095, USA \\
\texttt{chenzx19@cs.ucla.edu}
\AND
    Quanquan Gu \\
    Department of Computer Science \\ University of California, Los Angeles \\
   Los Angeles,  CA 90095, USA \\
    \texttt{qgu@cs.ucla.edu} 
    \And
Sham M. Kakade \\
Kempner Institute at Harvard University \\
Harvard University \\
 Cambridge, MA 02138, USA \\
 \texttt{sham@seas.harvard.edu}
}

\begin{document}
\date{}
\maketitle

\begin{abstract}
The $k$-sparse parity problem is a classical problem in computational complexity and algorithmic theory, serving as a key benchmark for understanding computational classes. In this paper, we solve the $k$-sparse parity problem with sign stochastic gradient descent, a variant of stochastic gradient descent (SGD) on two-layer fully-connected neural networks. We demonstrate that this approach can efficiently solve the $k$-sparse parity problem on a $d$-dimensional hypercube ($k\le O(\sqrt{d})$) with a sample complexity of $\tilde{O}(d^{k-1})$ using $2^{\Theta(k)}$ neurons, matching the established $\Omega(d^{k})$ lower bounds of Statistical Query (SQ) models\footnote{The $\tilde{O}(d^{k-1})$ sample complexity implies $\tilde{O}(d^{k})$ query complexity, because each sample corresponds to $d$ scalar-valued queries.}. Our theoretical analysis begins by constructing a $\good$ neural network capable of correctly solving the $k$-parity problem. We then demonstrate how a trained neural network with sign SGD can effectively approximate this good network, solving the $k$-parity problem with small statistical errors. To the best of our knowledge, this is the first result that matches the SQ lower bound for solving $k$-sparse parity problem using gradient-based methods.
\end{abstract}

\section{Introduction}\label{sec:intro}
The $k$-parity problem, defined on a binary sequence of length $d$, is a fundamental problem in the field of computational complexity and algorithmic theory. This problem involves finding a subset of cardinality $k$ by assessing if the occurrence of $1$'s in this subset is even or odd. The complexity of the problem escalates as the parameter $k$ increases. Its significance, while evidently practical, is rooted in its theoretical implications; it serves as a vital benchmark in the study of computational complexity classes and has profound implications for our understanding of P versus NP \citep{vardy1997intractability,downey1999parametrized,dumer2003hardness} and other cornerstone questions in computational theory \citep{blum2005line,klivans2006toward}.  Furthermore, the $k$-parity problem's complexity underpins many theoretical models in error detection and information theory \citep{dutta2008tight}, and is instrumental in delineating the limitations and power of algorithmic efficiency \citep{farhi1998limit}. This paper tackles the $k$-sparse parity problem~\citep {daniely2020learning}, where the focus is on the parity of a subset with cardinality $k \ll d$. 

Recent progress in computational learning theory has focused on improving the sample complexity guarantees for learning $k$-sparse parity functions using stochastic gradient descent (SGD). Under the framework of the Statistical Query (SQ) model \citep{kearns1998efficient}, it has been established that learning the $k$-sparse parity function requires a minimum of $\Omega(d^k)$ queries \citep{barak2022hidden}, highlighting the challenge in efficiently learning these functions.  On the other hand, considerable effort has been devoted to establishing sample complexity upper bounds for the special XOR case ($k=2$), with notable successes including $O(d)$ sample complexity using infinite-width or exponential-width (i.e., $O(2^{d})$) two-layer neural networks trained via gradient flow \citep{wei2018regularization, chizat2020implicit, telgarsky2022feature}, and $O(d^2)$ sample complexity with polynomial-width networks via SGD \citep{ji2019polylogarithmic, telgarsky2022feature}. A significant advancement in solving the $2$-parity problem was recently introduced by \cite{glasgow2023sgd}.  They proved a sample complexity bound of $\tilde{O}(d)$ using a two-layer ReLU network with logarithmic width trained by SGD, thus matching the SQ lower bound when $k=2$.

In the general case of $k\geq 2$,  \citet{barak2022hidden} has made significant progress, achieving a sample complexity of $\tilde{O}(d^{k+1})$ with a network width requirement of $2^{\Theta(k)}$, which is independent of the input dimension $d$. Additionally, \citet{barak2022hidden} demonstrated that the neural tangent kernel (NTK)-based method \citep{jacot2018neural} requires a network of polynomial width $d^{\Omega(k)}$ to solve the $k$-parity problem. Recently, \citet{suzuki2023feature} achieved a sample complexity of $O(d)$ by the mean-field Langevin dynamics (MFLD) \citep{mei2018mean, hu2019mean}, 
which requires neural networks with an exponential width in $d$, i.e., $O(e^{d})$, and an exponential number of iterations (i.e., $O(e^{d})$ ) to converge. Thus, their method is not computationally efficient and does not match the SQ lower bound. Notably, \cite{abbe2023sgd} introduced the leap-$k$ function for binary and Gaussian sequences, which extends the scope of the $k$-parity problem. They also proved  Correlational Statistical Query (CSQ) \citep{kearns1998efficient, bshouty2002using} 
lower bounds for learning leap-$k$ function for both Gaussian and Boolean inputs. In detail, they proved CSQ lower bounds of $\Omega(d^{k-1})$ for Boolean input and $\Omega(d^{k/2})$ for Gaussian input, which suggests that learning from Boolean input can be substantially harder than learning from Gaussian input. They also proved that SGD
can learn low dimensional target functions with Gaussian isotropic data and
2-layer neural networks using $n \gtrsim d^{\mathrm{Leap} - 1}$ examples. However, their upper bound analysis is based on the assumption that the input data $\xb$ follows a Gaussian distribution and relies on Hermite polynomials, making it unclear how to extend it to analyze Boolean input. Based on the above review of existing literature, it raises a natural but unresolved question: 
\begin{center}
    \emph{Is it possible to match the statistical query lower bound for $k$-sparse parity problems with stochastic gradient descent?} 
\end{center}

In this paper, we give an affirmative answer to the above question. In particular, we consider the standard $k$-sparse parity problem, where the input $\xb$ is drawn from a uniform distribution over $d$-dimensional hypercube $\mathrm{Unif}(\{-1,1\}^{d})$. 
Our approach involves training two-layer fully-connected neural networks with \CC{$m = 2^{\Theta(k)}$} width using sign SGD \citep{bernstein2018signsgd} 
with batch size $B=O(d^{k-1}\polylog(d))$.  We prove that the neural network trained by SGD can achieve a constant-order positive margin with high probability after $T=O(\log d)$ iterations. Therefore, the total number of examples required in our approach is $n = BT = \tilde{O}(d^{k-1})$. Thus, the total number of scalar-valued queries required in our paper is $m\cdot d \cdot n = 2^{\Theta(k)}\cdot d \cdot (d^{k-1}\cdot \text{polylog}d) = 2^{\Theta(k)}d^{k}\cdot \text{polylog}d$, where $m$ is the number of neurons, $d$ is the input dimension, and $n$ is the total number of fresh examples seen by the algorithm\footnote{Note that $m\cdot d$ is the total number of scalar-valued queries used for one example.}. \citet{abbe2023sgd} also proved a CSQ lower bound $\Omega(d^{k})$ 
for learning $d$-dimensional $k$-parity problems, which implies the sample complexity lower bound $n \gtrsim d^{k-1}$. Thus, our sample complexity result also matches the CSQ lower bound in \citet{abbe2023sgd}.

\subsection{Our Contributions}
The Statistical Query (SQ) lower bound indicates that, regardless of architecture, SGD 
requires a query complexity of $\Omega(d^k)$ for learning $k$-sparse $d$-dimensional parities under a constant noise level. We push the sample complexity frontier of $k$-sparse parity problem to $\tilde{O}(d^{k-1})$ via SGD, specifically with online stochastic sign gradient descent. 
Our main result is stated in the following informal theorem:

\begin{theorem}[Informal]
For a two-layer fully-connected neural networks of width $2^{\Theta(k)}$, online sign SGD with batch size $\tilde{O}(d^{k-1})$ can find a solution to the $k$-parity problem with a small test error within $O(k\log d)$ iterations.
\end{theorem}
The above theorem improves the sample complexity in \citet{barak2022hidden} from $\tilde{O}(d^{k+1})$ to $\tilde{O}(d^{k-1})$. Moreover, the total number of queries required is $\tilde{O}(d^{k})$, which matches the SQ/CSQ lower bound up to logarithmic factors. Additionally, under the standard basis setting, our result matches the sample complexity in \citet{glasgow2023sgd} for solving the XOR (i.e., $2$-parity) problem with sign SGD. It is worth noting that our result only requires two-layer fully connected neural networks with $2^{\Theta(k)}$ width and sign SGD training with $O(k\log d)$ iterations, which gives a computationally efficient algorithm. Finally, we empirically verify our theory in Appendix~\ref{sec:exp}, showcasing the efficiency and efficacy of our approach.

\paragraph{Notation.} We use $[N]$ to denote the index set $\{1, \dots, N\}$. We use lowercase letters, lowercase boldface letters, and uppercase boldface letters to denote scalars, vectors, and matrices, respectively. For a vector $\vb=(v_1,\cdots,v_d)^{\top}$, we denote by $\|\vb\|_{2}:=(\sum_{j=1}^{d}v_{j}^{2})^{1/2}$ its $L_2$ norm. For a vector $\vb=(v_1,\cdots,v_d)^{\top}$, we denote by $\vb_{[i_1:i_2]}:=(v_{i_1},\cdots,v_{i_2})^{\top}$ its truncated vector ranging from the $i_1$-th coordinate to the $i_2$-th coordinate. We denote by $\zero$ a vector of all zeros. For two sequence $\{a_k\}$ and $\{b_k\}$, we denote $a_k=O(b_k)$ if $|a_k|\leq C|b_k|$ for some absolute constant $C$, denote $a_k=\Omega(b_k)$ if $b_k=O(a_k)$, and denote $a_k=\Theta(b_k)$ if $a_k=O(b_k)$ and $a_k=\Omega(b_k)$. We also denote $a_k=o(b_k)$ if $\lim|a_k/b_k|=0$. We use $\tilde{O}(\cdot)$ and $\tilde{\Omega}(\cdot)$ to omit logarithmic terms in the notation. Finally, we denote $x_n=\poly(y_n)$ if $x_n=\cO( y_n^{D})$ for some positive constant $D$, and $x_n = \polylog(y_n)$ if $x_n= \poly( \log (y_n))$. 

\section{Related Work}
\paragraph{XOR Problem.} The performance of two-layer neural networks in the task of learning $2$-parity has been the subject of extensive research in recent years. \cite{wei2018regularization,chizat2020implicit,telgarsky2022feature} employed margin techniques to establish the convergence toward a global margin maximization solution, utilizing gradient flow and sample complexity of $O(d)$. Notably, \cite{wei2018regularization} and \cite{chizat2020implicit} employed infinite-width neural networks, while \cite{telgarsky2022feature} employed a more relaxed width condition of $O(d^{d})$. A significant breakthrough in this domain was achieved by \cite{glasgow2023sgd}, who demonstrated a sample complexity of $\tilde{O}(d)$ by employing SGD in conjunction with a ReLU network of width $\polylog(d)$. Furthermore, several other studies have shown that when input distribution follows Gaussian distribution, neural networks can be effectively trained to learn the XOR cluster distribution \citep{frei2022random, meng2023benign, xu2023benign}. A comparison between the results in this paper and those of related work solving the XOR problem can be found in Table~\ref{table:1}.

\paragraph{$k$-parity Problem.} The challenge of training neural networks to learn parities has been explored in previous research from diverse angles.  Several papers studied learning the $k$-parity function by using two-layer neural networks. \cite{daniely2020learning} studied learning $k$-parity function by applying gradient descent on the population risk (infinite sample size). Notably, \cite{barak2022hidden} presented both empirical and theoretical evidence that the $k$-parity function can be effectively learned using SGD and a neural network of constant width, demonstrating a sample complexity of $O(d^{k+1})$ and query complexity of $O(d^{k+2})$. \CC{\citet{edelman2024pareto} demonstrated that sparse initialization and increased network width lead to
improvements in sample efficiency. Specifically, they showed that the sample complexity can be reduced at the cost of increasing the width. However, the best statistical query complexity they can achieve is $O(d^{k+2})$, which is the same as  that in \citet{barak2022hidden}. }\cite{suzuki2023feature} reported achieving a sample complexity of $O(d)$ by employing mean-field Langevin dynamics (MFLD). 
Furthermore, \cite{abbe2022merged} and \cite{abbe2023sgd} introduced a novel complexity measure termed ``leap'' and established that leap-$k$ (with $k$-parity as a special case) functions can be learned through SGD with a sample complexity of $\tilde{O}(d^{\max(k-1,1)})$.
Additionally, \cite{abbe2023provable} demonstrated that a curriculum-based noisy-GD (or SGD) approach could attain a sample complexity of $O(d)$, provided the data distribution comprises a mix of sparse and dense inputs. The conditions outlined in this paper are compared to those from related work involving uniform Boolean data distribution, as presented in Table~\ref{table:2}.

\newcolumntype{g}{>{\columncolor{LightCyan}}c}

\begin{table*}[!tb]

\centering
\resizebox{\textwidth}{!}{%
\begin{tabular}{ccccccc}
\toprule
    & Activation & Loss & \multirow{2}{*}{Algorithm} &  Width ($m$)         & Sample ($n$)    &  Iterations ($t$)  
\\
    & Function & Function & &  Requirement    & Requirement   & to Converge 
\\ 
\midrule 
Theorem 2.1 & & & &\\
\multirow{-1}{*}{\citep{wei2018regularization}} & \multirow{-2}{*}{ReLU} & \multirow{-2}{*}{logistic} & \multirow{-2}{*}{WF with Noise} &\multirow{-2}{*}{$\infty$}&\multirow{-2}{*}{$d/\epsilon$}& \multirow{-2}{*}{$\infty$}\\
\midrule 
Theorem 8 & & & &\\
\multirow{-1}{*}{\citep{chizat2020implicit}} & \multirow{-2}{*}{2-homogenous} & \multirow{-2}{*}{logistic/hinge} & \multirow{-2}{*}{WF} & \multirow{-2}{*}{$\infty$}&\multirow{-2}{*}{$d/\epsilon$}& \multirow{-2}{*}{$\infty$}\\
\midrule 
Theorem 3.3& & & &\\
\small{\citep{telgarsky2022feature}}  & \multirow{-2}{*}{ReLU} & \multirow{-2}{*}{logistic} & \multirow{-2}{*}{scalar GF}&\multirow{-2}{*}{$d^{d}$}&\multirow{-2}{*}{$d/\epsilon$}& \multirow{-2}{*}{$d/\epsilon$}\\

\midrule

Theorem 3.2& & & &\\
\multirow{-1}{*}{\citep{ji2019polylogarithmic}} & \multirow{-2}{*}{ReLU} & \multirow{-2}{*}{logistic} &\multirow{-2}{*}{SGD}&\multirow{-2}{*}{$d^{8}$}&\multirow{-2}{*}{$d^{2}/\epsilon$}& \multirow{-2}{*}{$d^{2}/\epsilon$}\\
\midrule
Theorem 2.1& & & &\\
\small{\citep{telgarsky2022feature}} & \multirow{-2}{*}{ReLU } & \multirow{-2}{*}{logistic} & \multirow{-2}{*}{SGD}&\multirow{-2}{*}{$d^{2}$}&\multirow{-2}{*}{$d^{2}/\epsilon$}& \multirow{-2}{*}{$d^{2}/\epsilon$}\\
\midrule
Theorem 3.1 & & & &\\
\multirow{-1}{*}{\citep{glasgow2023sgd}} & \multirow{-2}{*}{ReLU} & \multirow{-2}{*}{logistic} & \multirow{-2}{*}{SGD} &\multirow{-2}{*}{$\mathrm{polylog}(d)$}& \multirow{-2}{*}{$d\cdot \mathrm{polylog}(d)$}& \multirow{-2}{*}{$\mathrm{polylog}(d)$}\\
\midrule
\rowcolor{LightCyan}
& & & & & &\\
\rowcolor{LightCyan}
\multirow{-2}{*}{Ours} & \multirow{-2}{*}{$x^2$} & \multirow{-2}{*}{correlation} &\multirow{-2}{*}{Sign SGD}&\multirow{-2}{*}{$O(1)$}&\multirow{-2}{*}{$d\cdot \mathrm{polylog}(d)$}& \multirow{-2}{*}{$\log d$}\\
\bottomrule
\end{tabular}
}
\caption{%
Comparison of existing works on the XOR ($2$-parity) problem. We mainly focus on the dependence on the input dimension $d$ and test error $\epsilon$ and treat other arguments as constant. Here WF denotes Wasserstein flow technique from the mean-field analysis, and GF denotes gradient flow. The sample requirement and convergence iteration in both \citet{glasgow2023sgd} and our method do not explicitly depend on the test error $\epsilon$. Instead, the dependence on $\epsilon$ is implicitly incorporated within the condition for $d$. Specifically, our approach requires that $d \geq C\log^2(2m/\epsilon)$ while \citet{glasgow2023sgd} requires $d  \geq \exp((1/\epsilon)^C)$ where $C$ is a constant.
}\label{table:1}. 
\end{table*}

\begin{table*}[!tb]
\centering
\resizebox{\textwidth}{!}{%
\begin{tabular}{ccccccc}
\toprule
    & Activation & Loss & \multirow{2}{*}{Algorithm} &  Width ($m$)         & Sample ($n$)    &  Iterations ($t$)  
\\
    & Function & Function & &  Requirement    & Requirement   & to Converge 
\\ 
\midrule
Theorem 4 & & & &\\
\multirow{-1}{*}{\citep{barak2022hidden}} & \multirow{-2}{*}{ReLU} & \multirow{-2}{*}{hinge} & \multirow{-2}{*}{SGD} &\multirow{-2}{*}{$2^{\Theta(k)}$}& \multirow{-2}{*}{$d^{k+1}\cdot\log(d/\epsilon)/\epsilon^2$}& \multirow{-2}{*}{$d/\epsilon^{2}$}\\
\midrule
Theorem 4 & & & &\\
\multirow{-1}{*}{\citep{edelman2024pareto}} & \multirow{-2}{*}{ReLU} & \multirow{-2}{*}{hinge} & \multirow{-2}{*}{SGD} &\multirow{-2}{*}{$(d/s)^{k}$}& \multirow{-2}{*}{$(s/k)^{k-1}d^{2}\log (d) /\epsilon^2 $}& \multirow{-2}{*}{$1/\epsilon^{2}$}\\
\midrule
Corollary 1 & & & &\\
\citep{suzuki2023feature} & \multirow{-2}{*}{Variant of Tanh} & \multirow{-2}{*}{ logistic} & \multirow{-2}{*}{MFLD} &\multirow{-2}{*}{$e^{d}$}& \multirow{-2}{*}{$d/\epsilon$}& \multirow{-2}{*}{$e^{d}$}\\
\midrule
\rowcolor{LightCyan}
& & & & & &\\
\rowcolor{LightCyan}
\multirow{-2}{*}{Ours} & \multirow{-2}{*}{$x^{k}$} & \multirow{-2}{*}{correlation} &\multirow{-2}{*}{Sign SGD}&\multirow{-2}{*}{$2^{\Theta(k)}$}&\multirow{-2}{*}{$d^{k-1}\cdot \mathrm{polylog}(d)$}& \multirow{-2}{*}{$\log d$}\\ 
\bottomrule
\end{tabular}
}
\caption{\CC{Comparison of existing works for the general $k$-parity problem, focusing primarily on the dimension $d$ and error $\epsilon$, treating other parameters as constants. $s$ in \citet{edelman2024pareto} is the sparsity of the initialization that satisfies $s>k$. }The activation function by \cite{suzuki2023feature} is defined as $h_{\wb}(\xb)=\Bar{R}[\tanh(\xb^{\top}\wb_1+w_2)+2\tanh(w_3)]/3$, where $\wb=(\wb_1,w_2,w_3)^{\top}\in\mathbb{R}^{d+2}$ and $\Bar{R}$ is a hyper-parameter determining the network's scale. For the sample requirement and convergence iteration, we focus on the dependency of $d, \epsilon$ and omit another terms. Our method's sample requirement and convergence iteration are independent of the test error $\epsilon$, instead relying on a condition for $d$ that implicitly includes $\epsilon$. Specifically, we require $d \geq C\log^2(2m/\epsilon)$. 
}\label{table:2}
\end{table*}

\section{Problem Setup}\label{sec:problem}
In this section, we introduce the $k$-sparse parity problem and the neural network we consider in this paper.  
\begin{definition}[$\boldsymbol{k}$-\textbf{parity}]\label{def:data}
Let each data point $(\xb,y)$ with $\xb \in \RR^{d}$ and $y\in\{-1,1\}$ be generated from the following distribution $\cD_{A}$, where $A$ is a non-empty set satisfying $A \subseteq [d]$: 
\begin{enumerate}[leftmargin=*,nosep]
    \item $x_{j} \sim \{-1,1\}$ as a uniform random bit for $j \in [d]$.
    
    \item The label $y$ is generated as $\Pi_{j \in A}x_j$. 
\end{enumerate}
The $k$-parity problem with dimension $d$ is defined as the task of recovering $A$, where $|A|=k$, using samples from $\cD_{A}$. 

\end{definition}

Without loss of generality, we assume that $A = \{1,\ldots, k\}$ if $|A| = k$. Under this assumption, we denote $\mathcal{D}_{A}$ by $\mathcal{D}$ for simplicity. This $k$-parity problem in Definition~\ref{def:data} is a classical one, which has been studied by \citet{daniely2020learning, barak2022hidden} using neural network learning. When restricted to the $2$-parity function, the problem is reduced to the XOR problem \citep{wei2018regularization}.

\noindent\textbf{Two-layer Neural Networks.} We consider a two-layer fully-connected neural network, which is defined as follows:
\begin{align}
f(\Wb,\xb)  &= \sum_{r=1}^m a_{r}\sigma(\la\wb_{r},\xb\ra), \label{eq:MLP}
\end{align}
where $m$ is the number of neurons. Here, we employ a polynomial activation function defined by $\sigma(z) = z^k$. The term $\wb_{r} \in \mathbb{R}^{d}$ represents the weight vector for the $r$-th neuron, and $\Wb$ denotes the aggregate of all first-layer model weights. The second-layer weights $a_{r}$'s are sampled uniformly from the set $\{-1, 1\}$ and fixed during training.

\paragraph{Algorithm.} We train the above neural network model by minimizing the correlation loss function:
\begin{align*}
    L_{\cD}(\Wb) 
    &= \EE_{(\xb,y)\sim \cD} \ell[ y \cdot f(\Wb, \xb) ],
\end{align*}
where $\ell(z)=1-z$. We consider binary initialization with $\wb_{r}^{(0)}\sim\mathrm{Unif}(\{\pm1\}^{d})$, which is widely used for neural networks solving parity problem \citep{barak2022hidden,abbe2022non}. We then employ the stochastic sign gradient descent with constant step size and weight decay (i.e., $\ell_2$ norm regularization on the first-layer weights) to minimize the correlation loss as follows:
\begin{align*}
\Wb^{(t+1)}=(1-\lambda\eta)\Wb^{(t)}-\eta\cdot\Tilde{\sign}\Big(\frac{\partial L^{(t)}}{\partial\Wb}\Big),
\end{align*}
where $\lambda>0$ is the weight decay parameter, $\eta>0$ is the step size, and $\Tilde{\sign}(x)$ is the modified sign function defined as:
\begin{figure}
    \centering
    \includegraphics[width=1.0\textwidth]{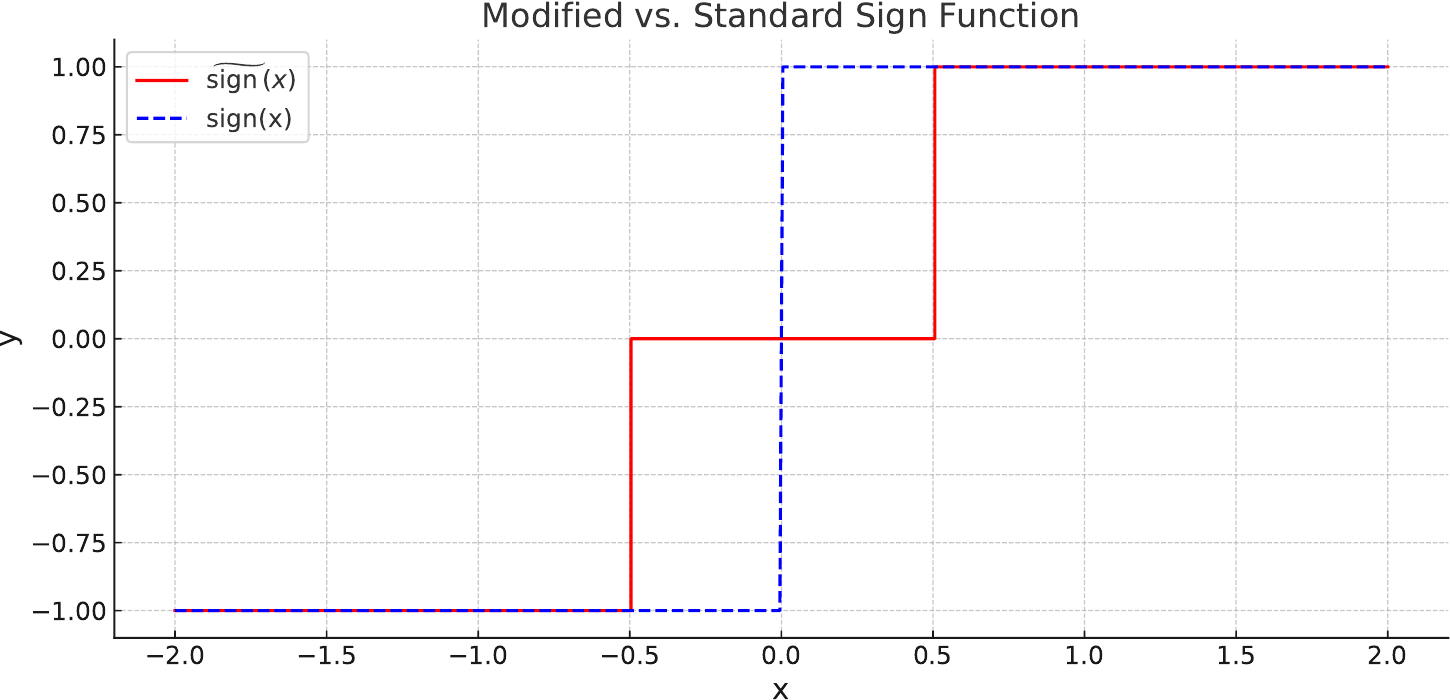}
    \caption{The plot above illustrates the comparison between the modified sign function $\tilde{\sign}(x) (\rho=0.5)$ and the standard sign function $\sign(x)$. The $\tilde{\sign}(x)$ function introduces a `dead zone' between $-\rho$ and $\rho$ where the function value is zero, which is not present in the standard sign function. This modification effectively creates a threshold effect, only outputting non-zero values when the input $x$ exceeds the specified bounds of $\rho$ in either direction.}
    \label{fig:enter-label}
\end{figure}
\begin{align*}
    \Tilde{\sign}(x)=\sign(x)\cdot\ind_{\{|x|\geq\rho\}}=\begin{cases}
        1,&\text{ for }x\geq\rho,\\
        0,&\text{ for }-\rho<x<\rho,\\
        -1,&\text{ for }x\leq-\rho.
    \end{cases}
\end{align*}
Here, $\rho>0$ is a threshold parameter. In this context, $L^{(t)}$ is computed using a randomly sampled online batch $S_{t}$ with batch size $|S_{t}|=B$:
\begin{align*}
    L^{(t)}=\frac{1}{B}\sum_{(\xb,y)\in S_{t}}\ell[ y \cdot f(\Wb^{(t)}, \xb) ].
\end{align*}
Consequently, the update rule for each $\wb_{r}$ is given by:
\begin{equation}\label{eq:w update}
    \wb_{r}^{(t+1)}=(1-\lambda\eta)\wb_{ r}^{(t)}+\eta\cdot\Tilde{\sign}\bigg(\frac{1}{B}\sum_{(\xb,y)\in S_{t}}\sigma'(\la\wb_{r}^{(t)},\xb\ra)\cdot a_r y\xb\bigg),
\end{equation}
where $\Tilde{\sign}$ is applied on an element-wise basis.

\begin{remark}
\CC{Sign SGD has been previously studied in \citet{riedmiller1993direct, bernstein2018signsgd}. Recently, it has become increasingly popular and has been utilized in adaptive optimizers for training large models \citep{chen2024symbolic, liu2023sophia}. Previous studies \citep{balles2018dissecting, bernstein2018signsgd, zou2021understanding} have demonstrated that Sign SGD behaves similarly to Adam when using sufficiently small step sizes or small moving average parameters, $\beta_1$ and $\beta_2$.} In our work, the choice of sign SGD over standard SGD stems primarily from our adoption of the polynomial activation function $\sigma(z) = z^k$. As later explained in Section~\ref{sec:main}, this specific activation function is pivotal in constructing a neural network that accurately tackles the $k$-parity problem. However, it introduces a trade-off: the gradient's dependency on the weights becomes polynomial rather than linear. Sign SGD addresses this issue by normalizing the gradient, ensuring that all neurons progress uniformly towards identifying the parity. Moreover, incorporating a threshold within the sign function plays a crucial role as it effectively nullifies the gradient of noisy coordinates. This, together with weight decay, aids in reducing noise, thereby enhancing the overall performance of the network. 
\end{remark}

\section{Main Results}\label{sec:main}
In this section, we begin by demonstrating the capability of the two-layer fully connected neural network~\eqref{eq:MLP} to classify all examples correctly. Specifically, we construct the following $\good$ network:
\begin{equation}\label{eq:optimal}
    f(\Wb^{*},\xb)=\sum_{r=1}^{2^{k}}a_r^{*}\sigma(\la\wb_{r}^{*},\xb\ra),
\end{equation}
where $\big\{\wb_{r,[1:k]}^{*}\big|r\in[2^{k}]\big\}=\{\pm1\}^{k}$, $a_r^{*}=\prod_{j=1}^{k}\sign(w_{r,j}^{*})$ and $\wb_{r,[k+1:d]}^{*}=\boldsymbol{0}_{d-k}$ for any $r\in[2^{k}]$. Notably, leveraging the inherent symmetry within our neural network model, we can formally assert the following proposition: $y f(\Wb^*,\mathbf{x}) = y' f(\Wb^*,\mathbf{x'})$ for any $(y, \xb)$ and $(y', \xb')$ generated from $\cD_{A}$. The subsequent proposition demonstrates the precise value of the margin.
\begin{proposition}\label{prop:symmetry}
For any data point $(\xb, y)$ generated from the distribution $\cD_{A}$, it holds that
\begin{align}
y f(\Wb^*,\xb)= k!\cdot 2^k. \label{eq:comb}  
\end{align}
\end{proposition}
\begin{proof}
Given a $(y, \xb) \in \cD_{A}$, we have that $y = \Pi_{i=1}^{k}x_i$. We divide the neurons into $(k+1)$ groups $\Omega_{i}, i \in \{0, \ldots, k\}$. A neuron $r\in \Omega_{i}$ if and only if $\sum_{s=j}^{k}\ind(\wb_{r}^{*} = x_{j}) = i$.
Then we have that 
\begin{align*}
y f(\Wb^*,\xb) &= \sum_{r=1}^{2^{k}}(y\cdot a_r^{*})\cdot \sigma(\la\wb_{r}^{*},\xb\ra)\\
&= \sum_{i=0}^{k}\sum_{r\in \Omega_i}(y\cdot a_r^{*})\cdot \sigma(\la\wb_{r}^{*},\xb\ra) \\
&= \sum_{i=0}^{k}\sum_{r\in \Omega_i}\bigg(\prod_{j=1}^{k}\sign(x_{j})\cdot \prod_{j=1}^{k}\sign(w_{r,j}^{*})\bigg)\cdot \sigma(\la\wb_{r}^{*},\xb\ra)\\
&= \sum_{i=0}^{k}{k\choose i}(-1)^{i}\sigma(k-2i)\\
&=k!\cdot 2^k,
\end{align*}
where the third equality is due to the fact that $y = \prod_{j=1}^{k}\sign(x_{j})$ and  $a_{r}^{*} = \prod_{j=1}^{k}\sign(w_{r,j}^{*})$, the fourth equality is due to the definition of $\Omega_i$, the last equality holds because $\sigma$ is $k$-th order polynomial activation function and Lemma~\ref{lm:auxi}.
\end{proof}
Therefore, we can conclude that for any $(\xb,y)$, we have $yf(\Wb^*,\xb)=k!\cdot 2^k>0$. We will demonstrate in the next section that training using large batch size online SGD, as long as Condition~\ref{Cond} is met, will lead to the trained neural network $f(\Wb^{(T)},\xb)$ approximating $(m/2^{k+1})\cdot f(\Wb^{*},\xb)$  effectively after $T=O(k\log(d))$ iterations. Our main theorem is based on the following conditions on the training strategy.
\begin{condition}\label{Cond}
Suppose there exists a sufficiently large constant $C$, such that the following conditions hold:
\begin{itemize}[leftmargin=*,nosep]
    \item Neural network width $m$ satisfies $m\geq C\cdot 5^{k}\log(1/\delta)$.
    \item Dimension is sufficiently large: $d\geq C\log^{2}(2m/\epsilon)$. 
    \item Online SGD batch size $B\geq C2^{k}((k-1)!)^{-2}d^{k-1}\log^{k-1}(16mdBT/\delta)\log^{2}(8mdT/\delta)$.
    \item Learning rate $\eta$ satisfies $\eta\leq C^{-1}$.
    \item Regularization parameter $\lambda$ is taken as $\lambda=1$.
    \item The threshold $\rho$ for the modified sign function satisfies $\rho=0.1k!$. 
\end{itemize}
\end{condition}
In the $k$-parity problem, the label $y$ is determined by a set of $k$ bits. Consequently, the total count of distinct features is $2^{k}$, reflecting all possible combinations of these bits. The condition of $m$ is established to guarantee a roughly equal number of neurons within the $\good$ neuron class, each correctly aligned with distinct features. The condition of $d$ ensures that the problem is in a sufficiently high-dimensional setting. The condition of $m, d$ implies that $d \geq \Omega(\log^2 m) \geq \Omega(k^2)$, which is a mild requirement for the sparsity $k$. In comparison, \citet{barak2022hidden} requires $d \geq \Omega(k^{4})$ for neural networks solving $k$-parity problem. By Stirling's approximation, the condition of $B$ can be simplified to $B\geq \tilde{\Omega}\big(k(2e\log(16mdBT/\delta)d/k^2)^{k-1}\big)$. Therefore, the conditional batch size $B$ will exponentially increase as parity $k \leq O(\sqrt{d})$ goes up, which ensures that the stochastic gradient can sufficiently approximate the population gradient. Finally, the conditions of $\eta, \lambda$ ensure that gradient descent with weight decay can effectively learn the relevant features while simultaneously denoising the data.  Finally, the threshold condition $\rho$ increases as sparsity $k$ increases to accommodate the increase of the population gradient. Based on these conditions, we give our main result on solving the parity problem in the following theorem. 
\begin{theorem}\label{main:theorem}
Under Condition~\ref{Cond}, we run online SGD for iteration $T=\Theta\big(k\eta^{-1}\lambda^{-1}\log d\big)$ iterations. Then with probability at least $1-\delta$ we can find $\Wb^{(T)}$ such that 
\begin{align*}
\mathbb{P}\big(yf(\Wb^{(T)}, \xb) \geq \gamma m\big) \geq 1-\epsilon, \end{align*}
where $\gamma=0.25k!$ is a constant. 
\end{theorem}
Theorem~\ref{main:theorem} establishes that, under certain conditions, a neural network is capable of learning to solve the $k$-parity problem within $\Theta(k\eta^{-1}\lambda^{-1}\log d)$ iterations, achieving a population error of at most $\epsilon$. According to Condition~\ref{Cond}, the total number of examples utilized amounts to $BT = \tilde{O}(d^{k-1})$ given a polynomial logarithmic width requirement of $m = O(1)$ with respect to $d$.

\begin{remark}
While Theorem~\ref{main:theorem} works for fixed second-layer training, we demonstrate that comparable results can be obtained when the second layer of the network is simultaneously trained with a lower learning rate. Detailed results and further elaboration of this aspect are provided in Appendix~\ref{sec:D}. Our findings present a sample complexity of $\tilde{O}(d^{k-1})$, aligning with Conjecture 2 posited by \citet{abbe2023sgd}, which suggests a sample complexity lower bound of $\Tilde{\Omega}(d^{k-1})$. Besides, our results for the uniform Boolean distribution match the complexity achieved by \citet{abbe2023sgd} under the isotropic Gaussian scenario. Despite these similarities in outcomes, our methodology diverges significantly: we employ online sign SGD utilizing a large batch size of $\tilde{O}(d^{k-1})$ and conduct training over merely $\tilde{O}(1)$ iterations. In contrast, \citet{abbe2023sgd} implement projected online SGD with a minimal batch size of $1$, extending training over $\tilde{O}(d^{k-1})$ iterations. \citet{abbe2023sgd} also requires a two-phase training process for the first and second layer weights, requiring them to be trained separately.
\end{remark}

\section{Overview of Proof Technique}
In this section, we discuss the main ideas used in the proof. Based on these main ideas, the proof of our main Theorem~\ref{main:theorem} will follow naturally. The complete proofs of all the results are given in the appendix. Section~\ref{subsec:warmup} serves as a warmup by examining population sign gradient descent. Here, three pivotal ideas crucial to the proof of stochastic sign gradient descent are introduced:
\begin{enumerate}[leftmargin=*,nosep]
\item The impact of the initialization's positivity or negativity on the trajectory of neuron weights.
\item The divergence between feature coordinates and noise coordinates of different neurons.
\item How a trained neural network can effectively approximate the $\good$ neural network \eqref{eq:optimal}.
\end{enumerate}
Moving on to Section~\ref{subsec:sgd}, we delve into the analysis of sign SGD. Contrasting with population GD, the addition in SGD analysis involves accounting for the approximation error between the population gradient and the stochastic gradient. This consideration leads to the stipulation of the batch size $B$ outlined in Condition~\ref{Cond}.

\subsection{Warmup: Population Gradient Descent}\label{subsec:warmup}
For population gradient, we perform the following updates:
\begin{align*}
    \wb_{r}^{(t+1)}&=(1-\eta\lambda)\cdot\wb_{r}^{(t)}-\eta\cdot\Tilde{\sign}\big(\nabla_{\wb_r}L_{\cD}(\Wb^{(t)})\big),
\end{align*}
where $L_{\cD}(\Wb)=\EE_{(\xb,y)\sim\cD}[\ell(y,f(\Wb,\xb))]=1-\EE_{(\xb,y)\sim\cD}[yf(\Wb,\xb)]$. Then, the following coordinate-wise population gradient update rules hold:
\begin{align}
    w_{r,j}^{(t+1)}&=(1-\eta\lambda)w_{r,j}^{(t)}+\eta \cdot\Tilde{\sign}\Bigg(k!a_{r}\frac{w_{r,1}^{(t)}w_{r,2}^{(t)}\cdots w_{r,k}^{(t)}}{w_{r,j}^{(t)}}\Bigg), &&j\in[k],\\
    w_{r,j}^{(t+1)}&=(1-\eta\lambda)w_{r,j}^{(t)}, &&j\notin[k].\label{eq:noise change}
\end{align}
In the preceding discussion of the update rule, we have identified that the noise coordinates ($j\notin[k]$) exhibit exponential decay, characterized by a decay constant of $1-\eta\lambda$. To further dissect the dynamics of this system, we turn our attention to the behavior of feature coordinates. We categorize neurons into two distinct types based on their initial alignment: a neuron is classified as a $\good$ neuron if $a_r=\prod_{j=1}^{k}\sign(w_{r,j}^{(0)})$, and conversely, as a $\bad$ neuron if $a_r=-\prod_{j=1}^{k}\sign(w_{r,j}^{(0)})$. This distinction is pivotal, as it divides the neuron population into two distinct classes: $\good$ and $\bad$. Neurons in the $\good$ class are integral to the functionality of the final trained neural network, playing a significant role in its test accuracy. Conversely, neurons classified as $\bad$ tend to diminish in influence over the course of training, ultimately contributing minimally to the network's overall performance. For $\good$ neurons, the update rules for feature coordinates can be reformulated to
\begin{equation}\label{eq:good neuron}
\begin{aligned}
    \sign(w_{r,j}^{(0)})w_{r,j}^{(t+1)}&=(1-\eta\lambda)\sign(w_{r,j}^{(0)})w_{r,j}^{(t)}\\
    &\qquad+\eta \cdot\frac{\sign(w_{r,1}^{(0)}w_{r,1}^{(t)})\sign(w_{r,2}^{(0)}w_{r,2}^{(t)})\cdots \sign(w_{r,k}^{(0)}w_{r,k}^{(t)})}{\sign(w_{r,j}^{(0)}w_{r,j}^{(t)})}.
\end{aligned}
\end{equation}
For neurons classified as \bad, the update rules for feature coordinates can be rewritten as:
\begin{equation}\label{eq:bad neuron}
\begin{aligned}
    \sign(w_{r,j}^{(0)})w_{r,j}^{(t+1)}&=(1-\eta\lambda)\sign(w_{r,j}^{(0)})w_{r,j}^{(t)}\\
    &\qquad-\eta \cdot\frac{\sign(w_{r,1}^{(0)}w_{r,1}^{(t)})\sign(w_{r,2}^{(0)}w_{r,2}^{(t)})\cdots \sign(w_{r,k}^{(0)}w_{r,k}^{(t)})}{\sign(w_{r,j}^{(0)}w_{r,j}^{(t)})}.
\end{aligned}
\end{equation}
Comparing equations \eqref{eq:good neuron} and \eqref{eq:bad neuron}, it becomes apparent that the feature coordinates of $\good$ and $\bad$  neurons exhibit divergent behaviors. Consequently, the feature coordinates of $\good$ neurons will significantly outweigh those of $\bad$ neurons in the long term. With the regularization parameter $\lambda$ set as $1$ in Condition~\ref{Cond}, we derive the following lemma illustrating the divergent trajectories of population gradient descent for both neuron types:
\begin{lemma}\label{lm:feature change}
Under Condition~\ref{Cond}, for $\good$ neurons $r\in\Omega_{\mathrm{g}}:=\{r\in[m]:a_r=\prod_{j=1}^{k}\sign(w_{r,j}^{(0)})\}$, the feature coordinates will remain the same as initialization throughout the training:
\begin{align*}
    w_{r,j}^{(t)}=w_{r,j}^{(0)},&&\forall j\in[k], t\geq 0.
\end{align*}
For $\bad$  neurons $r\in\Omega_{\mathrm{b}}:=\{r\in[m]:a_r=-\prod_{j=1}^{k}\sign(w_{r,j}^{(0)})\}$, the feature coordinates will decay faster than noise coordiantes:
\begin{align*}
    0<\sign(w_{r,j}^{(0)})w_{r,j}^{(t+1)}\leq(1-\eta\lambda)\sign(w_{r,j}^{(0)})w_{r,j}^{(t)},&&\forall j\in[k], t\geq 0.
\end{align*}
\end{lemma}
According to \eqref{eq:noise change} and Lemma~\ref{lm:feature change}, after training $T=\Theta(k\eta^{-1}\lambda^{-1}\log(d))$ iterations, $\bad$ neurons and noise coordinates in $\good$ neurons diminish to a magnitude of $\Theta(1/\poly(d))$, as shown in Lemma~\ref{lm:exact order4}. In contrast, the feature coordinates of $\good$ neurons remain unchanged.
\begin{lemma}\label{lm:exact order4}
Under Condition~\ref{Cond}, for $T\geq (k+1)\eta^{-1}\lambda^{-1}\log(d)$, it holds that
\begin{align*}
    |w_{r,j}^{(T)}|&\leq d^{-(k+1)}, &&\forall r\in\Omega_{\mathrm{g}},j\in[d]\setminus[k],\\
    |w_{r,j}^{(T)}|&\leq d^{-(k+1)}, &&\forall r\in\Omega_{\mathrm{b}},j\in[d].
\end{align*}  
\end{lemma}
This leads to the following approximation for the trained neural network:
\begin{align*}
    f(\Wb^{(T)},\xb)&=\sum_{r=1}^{m}a_r\sigma(\la\wb_{r}^{(T)},\xb\ra)\approx\sum_{r\in\Omega_{\mathrm{g}}}a_r\sigma(\la\wb_{r}^{(T)},\xb\ra)\\
    &=\sum_{r\in\Omega_{\mathrm{g}}}\bigg(\prod_{j=1}^{k}\sign(w_{r,j}^{(0)})\bigg)\sigma(\la\wb_{r}^{(T)},\xb\ra)\approx\sum_{r\in\Omega_{\mathrm{g}}}\bigg(\prod_{j=1}^{k}\sign(w_{r,j}^{(0)})\bigg)\sigma(\la\wb_{r,[1:k]}^{(T)},\xb_{[1:k]}\ra).
\end{align*}
Under Condition~\ref{Cond}, the condition on $m$ ensures a balanced distribution of neurons across different initializations, approximately $m/2^{k+1}$, given $2^{k+1}$ kinds of possible initializations. This results in the trained neural network $f(\Wb^{(T)},\xb)$ closely approximating $(m/2^{k+1})\cdot f(\Wb^{*},\xb)$. 
\subsection{Stochastic Sign Gradient Descent}\label{subsec:sgd}

Transitioning from the trajectory trained by population gradient descent, this section delves into the dynamics under sign stochastic gradient descent (Sign SGD). We commence by presenting a lemma that estimates the approximation error between the population gradient and the stochastic gradient.
\begin{lemma}\label{lm:approx}
Under Condition~\ref{Cond}, with probability at least $1-\delta$ with respect to the online data generation, the stochastic gradient approximates the population gradient well:
\begin{align*}
\bigg|\frac{\partial L_{\cD}(\Wb^{(t)})}{\partial w_{r,j}}-\frac{\partial L^{(t)}}{\partial w_{r,j}}\bigg|\leq \epsilon_{1}\cdot\|\wb_r^{(t)}\|_{2}^{k-1},&&\forall t\in[0,T], r\in[m], j\in[d],
\end{align*}
where $L^{(t)}$ is the loss of randomly sampled online batch $S_t$ and 
\begin{align*}
  \epsilon_1 = \tilde{O}(B^{-1/2} + d^{(k-3)/2}B^{-1}).
\end{align*}
\end{lemma}

The choice of batch size $B=O(d^{k-1})$ in our algorithm is crucial for gradient concentration. According to Lemma~\ref{lm:approx}, the gap between stochastic gradient and population gradient is bounded by $\epsilon_1\cdot\|\wb_r\|_2^{k-1}$. At initialization, the absolute value of the population gradient on the signal coordinate is approximately $\tilde{O}(d^{-(k-1)/2})\cdot\|\wb_r\|_2^{k-1}$. To ensure the stochastic sign gradient matches the population sign gradient, the approximation error $\epsilon_1$ must be smaller than this value, which requires $\epsilon_1 = \tilde{O}(d^{-(k-1)/2})$. Solving $B^{-1/2} + d^{(k-3)/2}B^{-1} = d^{-(k-1)/2}$ yields our sufficient batch size.

Building upon this approximation guarantee from Lemma~\ref{lm:approx}, and considering the established order of $\rho,\epsilon_1$ and $\|\wb_r^{(t)}\|_{2}$, we arrive at an important corollary.

\begin{corollary}
Under Condition~\ref{Cond}, given the same initialization, with probability at least $1-\delta$, the stochastic sign gradient is the same as the population sign gradient:
\begin{align*}
    \Tilde{\sign}\bigg(\frac{\partial L_{\cD}(\Wb^{(t)})}{\partial w_{r,j}}\bigg)=\Tilde{\sign}\bigg(\frac{\partial L^{(t)}}{\partial w_{r,j}}\bigg),&&\forall t\in[0,T], r\in[m], j\in[d].
\end{align*}
\end{corollary}
This corollary suggests that, under identical initialization, the trajectory of a model trained using population gradient descent will, with high probability, align with the trajectory of a model trained using stochastic gradient descent.

\section{Conclusion and Future Work}
In our study, we have conducted a detailed analysis of the $k$-parity problem, investigating how sign Stochastic Gradient Descent (sign SGD) can effectively learn intricate features from binary datasets. Our findings reveal that sign SGD, when employed in two-layer fully-connected neural networks solving $k$-sparse parity problem, is capable of achieving a sample complexity $\tilde{O}(d^{k-1})$. Remarkably, this result matches the theoretical expectations set by the Statistical Query (SQ) model, underscoring the efficiency and adaptability of sign SGD. 

Looking ahead, an intriguing direction for future research is to explore the possibility of learning $k$-parity using SGD with even smaller queries that surpass the SQ lower bond, and understand whether more standard neural network architectures allow such improvement. This potential advancement could pave the way for developing more efficient algorithms capable of tackling complex problems with weaker data requirements. Another promising direction is to extend our results to non-isotropic data settings \citep{nitanda2024improved}, where sign gradient descent with momentum could be effective in handling the transformed feature space.

\section*{Acknowledgements}
We thank the anonymous reviewers for their helpful comments. YK, ZC, and QG are supported in part by the National Science Foundation CAREER Award 1906169, IIS-2008981, and the Sloan Research Fellowship. SK acknowledges: this work has been made possible in part by a gift from the Chan Zuckerberg Initiative Foundation to establish the Kempner Institute
for the Study of Natural and Artificial Intelligence; support from the Office of Naval Research under award N00014-22-1-2377, and the
National Science Foundation Grant under award \#IIS 2229881.

\bibliography{deeplearningreference}
\bibliographystyle{ims}

\newpage
\appendix

\section*{Limitations}
While our study provides valuable insights into the effectiveness of SGD in learning the $k$-parity problem, there are some limitations:
\begin{itemize}[leftmargin=*]
\item Our study focuses on sign gradient descent (Sign SGD). This approach normalizes the gradient, ensures uniform neuron updates toward identifying parity, and effectively nullifies noisy coordinate gradients. However, it's worth noting that data may be presented in non-standard or unknown coordinate systems, which could limit Sign SGD's effectiveness. To address this limitation, future work could explore alternatives such as normalized gradient descent with an adaptive learning rate or incorporating momentum into Sign SGD. 

\item Our analysis is primarily based on polynomial activation functions. While effective for the standard k-parity problem, extending our approach to other activation functions like sigmoid or ReLU presents challenges. This extension could potentially be achieved through polynomial function approximation. However, the main challenge lies in identifying an appropriate functional decomposition and accurately characterizing the approximation error during training.

\end{itemize}
By acknowledging these limitations, we aim to provide a transparent assessment of our work's scope and potential areas for future exploration.

\section{Experiments}\label{sec:exp}
In this section, we present a series of experiments designed to empirically validate the theoretical results established in our main theorem. The primary objectives of these experiments are to (1) assess the test accuracy of the trained neural network, thereby corroborating the theorem's results, and (2) verify the key lemmas concerning the behavior of $\good$ and $\bad$  neurons. Specifically, we aim to demonstrate that for $\good$ neurons, the feature coordinates remain largely unchanged from initialization while the noise coordinates decay exponentially. Conversely, for $\bad$  neurons, we expect both feature and noise coordinates to exhibit exponential decay.

\paragraph{Model.} We generated synthetic $k$-parity data based on Definition~\ref{def:data}. Each data point $(\xb,y)$ with $\xb\in\RR^{d}$ and $y\in\{\pm1\}$ is produced from distribution $\cD_{A}$, where $A$ specifically is taken as $[k]$. We utilized two-layer fully-connected neural networks with $m$ number of neurons. The network employs a polynomial activation function $\sigma(z)=z^{k}$. The first layer weights for the $r$-th neuron, $\wb_{r}^{(0)}\in\RR^{d}$, were initialized following a binary scheme, where $\wb_{r}^{(0)}\sim\text{Unif}(\{\pm1\}^{d})$. The second layer weights for the $r$-th neuron, $a_r$, is randomly initialized as $1$ or $-1$ with equal probability.

\paragraph{Computation Resources} The experiments are conducted on an A6000 server. As these are synthetic experiments, the requirement for computational resources is minimal.

\paragraph{Training.} Our model was trained to minimize the empirical correlation loss function, incorporating $L_2$ regularization with a regularization parameter set at $\lambda=1$. The training process utilized online stochastic sign gradient descent (Sign SGD) with a fixed step size $\eta$ and a predetermined batch size $B$. The principal metric for assessment was test accuracy. Our experiments were conducted considering parity $k \in \{2,3,4\}$.

For $k=2$, the model configuration included a data dimension of $d=8$, a hidden layer width of $m=12$, a total of $T=25$ epochs, a learning rate $\eta=0.1$, an online batch size of $B=64$, and a threshold for $\Tilde{\sign}$ set at $\rho=0.3$. In the case of $k=3$, we employed a data dimension of $d=16$, increased the hidden layer width to $m=48$, extended the training to $T=50$ epochs, adjusted the learning rate to $\eta=0.05$, used an online batch size of $B=256$, and set the threshold for $\Tilde{\sign}$ at $\rho=1$. For $k=4$, the model was further scaled up with a data dimension of $d=20$, a hidden layer width of $m=128$, a training epoch of $T=100$, a smaller learning rate of $\eta=0.02$, an online batch size of $B=2048$, and a threshold for $\Tilde{\sign}$ established at $\rho=3$.

\paragraph{Experimental Results.} The evaluation of test accuracy across various configurations is presented in Table~\ref{tab:test_acc}. Using neural networks with merely $2^{\Theta(k)}$ neurons, we observed high test accuracy for $k$-parity problem with $k\in\{2,3,4\}$, confirming the results of our main theorem (Theorem~\ref{main:theorem}). These empirical results validate the efficacy of our studied model architecture \eqref{eq:MLP} and training methodology in tackling the $k$-sparse parity problem.

To further validate our theoretical findings, we examined the change of feature and noise coordinates for the first neuron $\wb_{1}^{(t)}$ over multiple iterations, focusing specifically on the setting $k\in\{2,3,4\}$. Figures~\ref{fig:good2}, \ref{fig:bad2}, \ref{fig:good3}, \ref{fig:bad3}, \ref{fig:good4}, and \ref{fig:bad4} visually represent these trajectories. Our empirical findings reveal a consistent pattern: the feature coordinates ($w_{1,1}^{(t)}, \ldots, w_{1,k}^{(t)}$) of the neurons identified as $\good$ (with initialization satisfying $a_{r} = \prod_{j=1}^{k} w_{r,j}^{(0)}$) exhibit relative stability, while their noise coordinates ($w_{1,k+1}^{(t)}, \ldots$) show a decreasing trend over time. In contrast, the trajectories for neurons classified as $\bad$ (with initialization satisfying $a_{r} = -\prod_{j=1}^{k} w_{r,j}^{(0)}$) indicate a general reduction in all coordinate values.

These empirical observations support our theoretical analyses, as outlined in Lemma~\ref{lm:feature change} and Lemma~\ref{lm:exact order4}, showcasing the consistency between the theoretical foundations of our model and its practical performance. The disparities between $\text{good}$ and $\text{bad}$ neurons, as evidenced by their feature and noise coordinate behaviors, underscore the nuanced dynamics inherent in the learning process of $k$-parity problems.

\begin{table*}[ht!]
\centering
\begin{tabular}{cccc}
\toprule
$\boldsymbol{k}$ & $2$ & $3$ & $4$ \\
\midrule
\textbf{Test Accuracy (\%)} & $99.69\% \pm 0.29\%$ & $97.75\% \pm 1.37\%$ & $96.89\% \pm 0.44\%$ \\
\bottomrule
\end{tabular}
\caption{Test accuracy for solving $k$-sparse parity problem with $k \in \{2, 3, 4\}$, averaged over 10 runs. }
\label{tab:test_acc}
\end{table*}


\begin{figure}[ht!]
\centering
\begin{minipage}{.48\textwidth}
  \centering
  \includegraphics[width=1\linewidth]{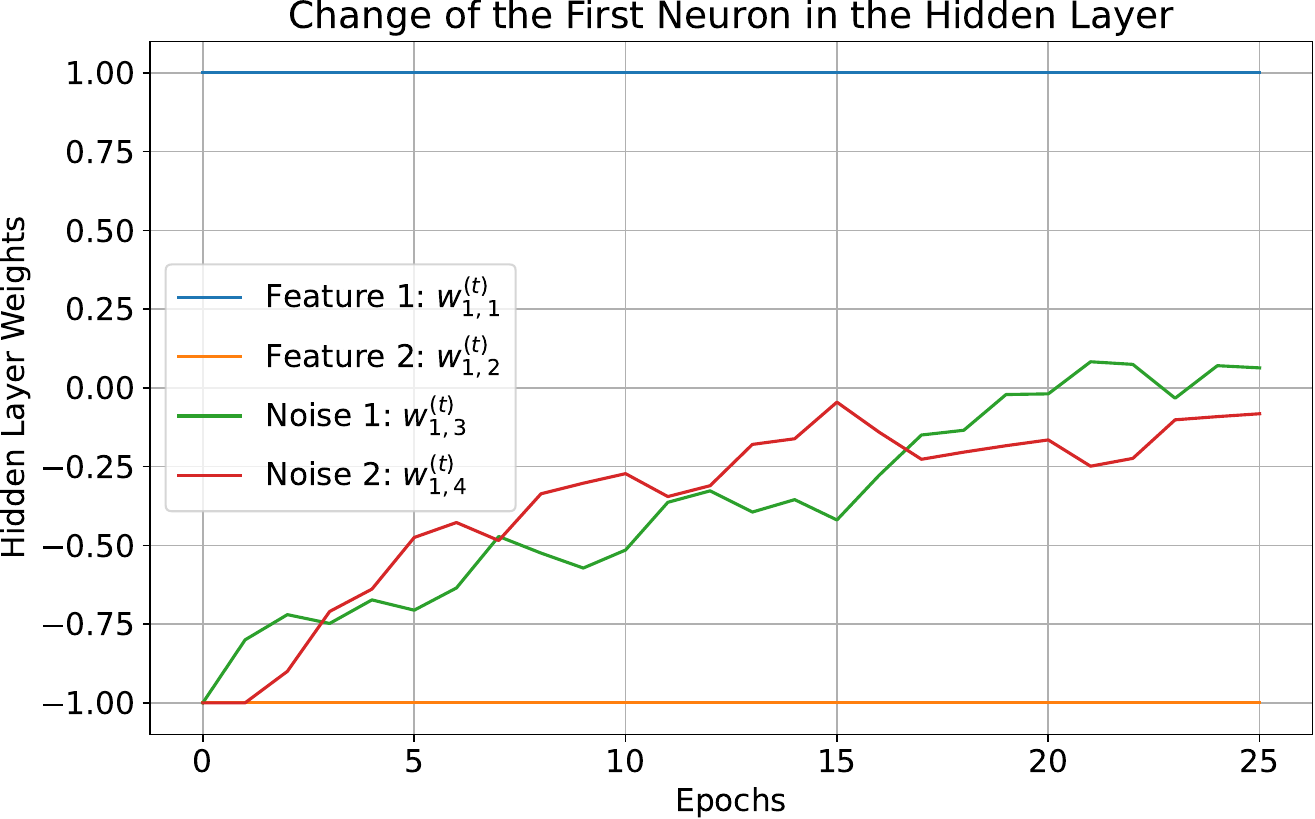}
  \captionof{figure}{Illustration of a $2$-parity $\good$ neuron with initial weights $w_{1,1}^{(0)} = 1$, $w_{1,2}^{(0)} = -1$, and $a_{1} = -1$.}
  \label{fig:good2}
\end{minipage}%
\hspace{0.04\textwidth}%
\begin{minipage}{.48\textwidth}
  \centering
  \includegraphics[width=1\linewidth]{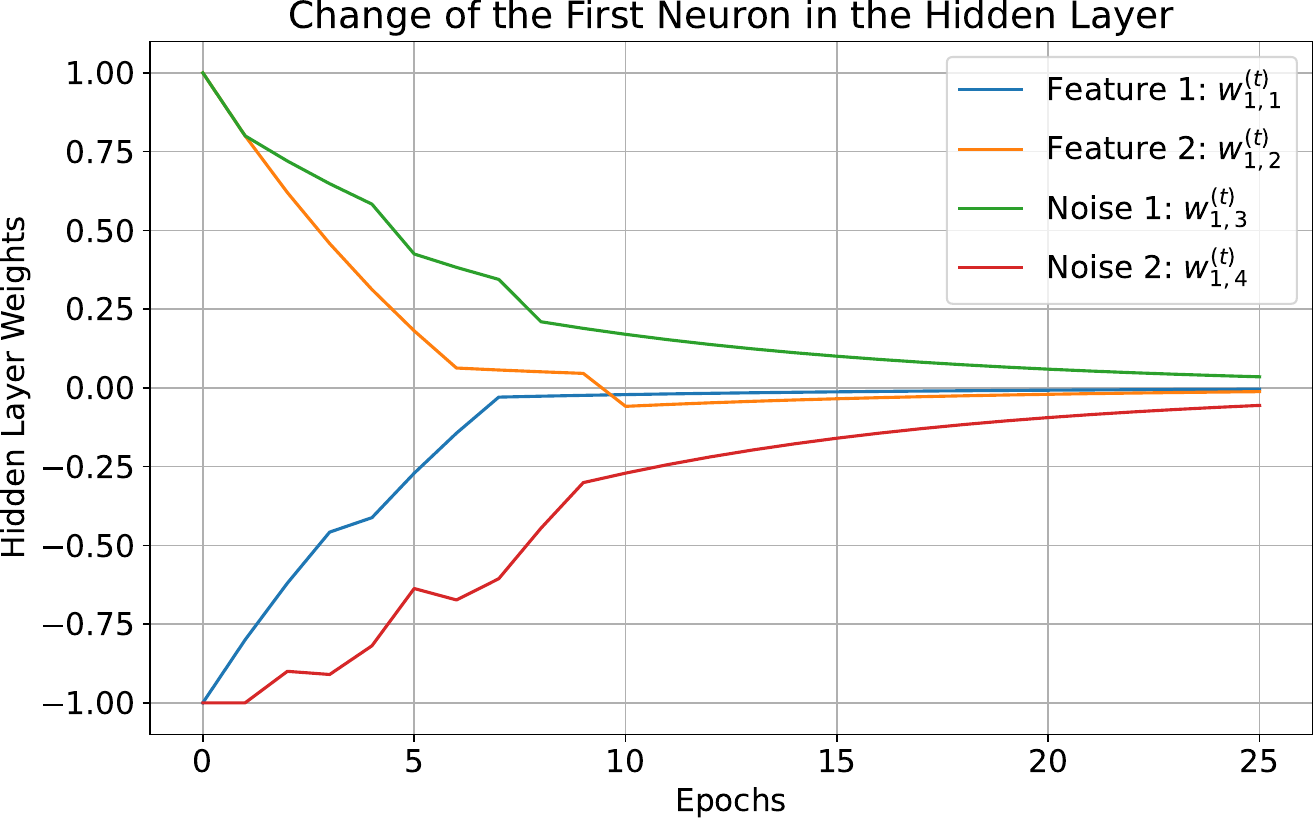}
  \captionof{figure}{Illustration of a $2$-parity $\bad$ neuron with initial weights $w_{1,1}^{(0)} = -1$, $w_{1,2}^{(0)} = 1$, and $a_{1} = 1$.}
  \label{fig:bad2}
\end{minipage}
\end{figure}

\begin{figure}[ht!]
\centering
\begin{minipage}{.48\textwidth}
  \centering
  \includegraphics[width=1\linewidth]{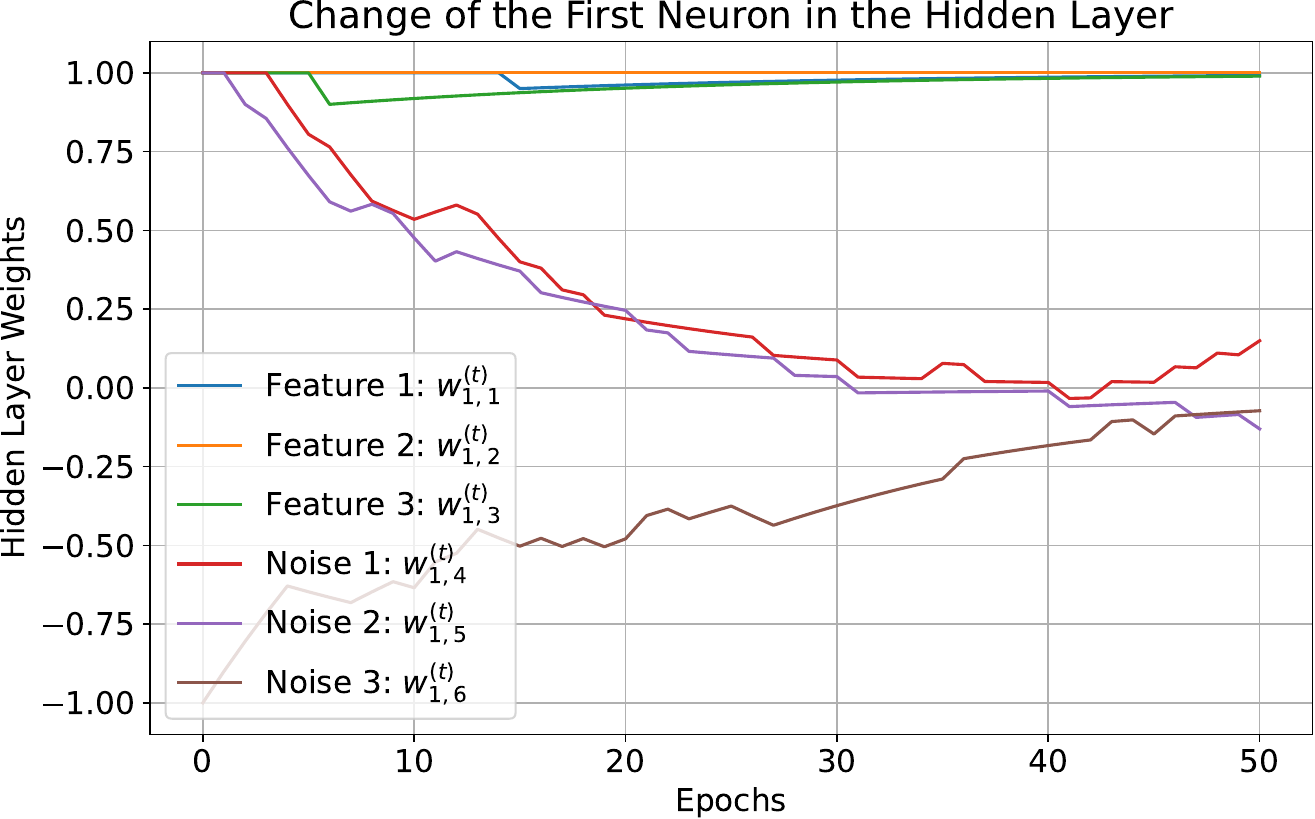}
  \captionof{figure}{Illustration of a $3$-parity $\good$ neuron with initial weights $w_{1,1}^{(0)} = 1$, $w_{1,2}^{(0)} = 1$, $w_{1,3}^{(0)} = 1$, and $a_{1} = 1$.}
  \label{fig:good3}
\end{minipage}%
\hspace{0.04\textwidth}%
\begin{minipage}{.48\textwidth}
  \centering
  \includegraphics[width=1\linewidth]{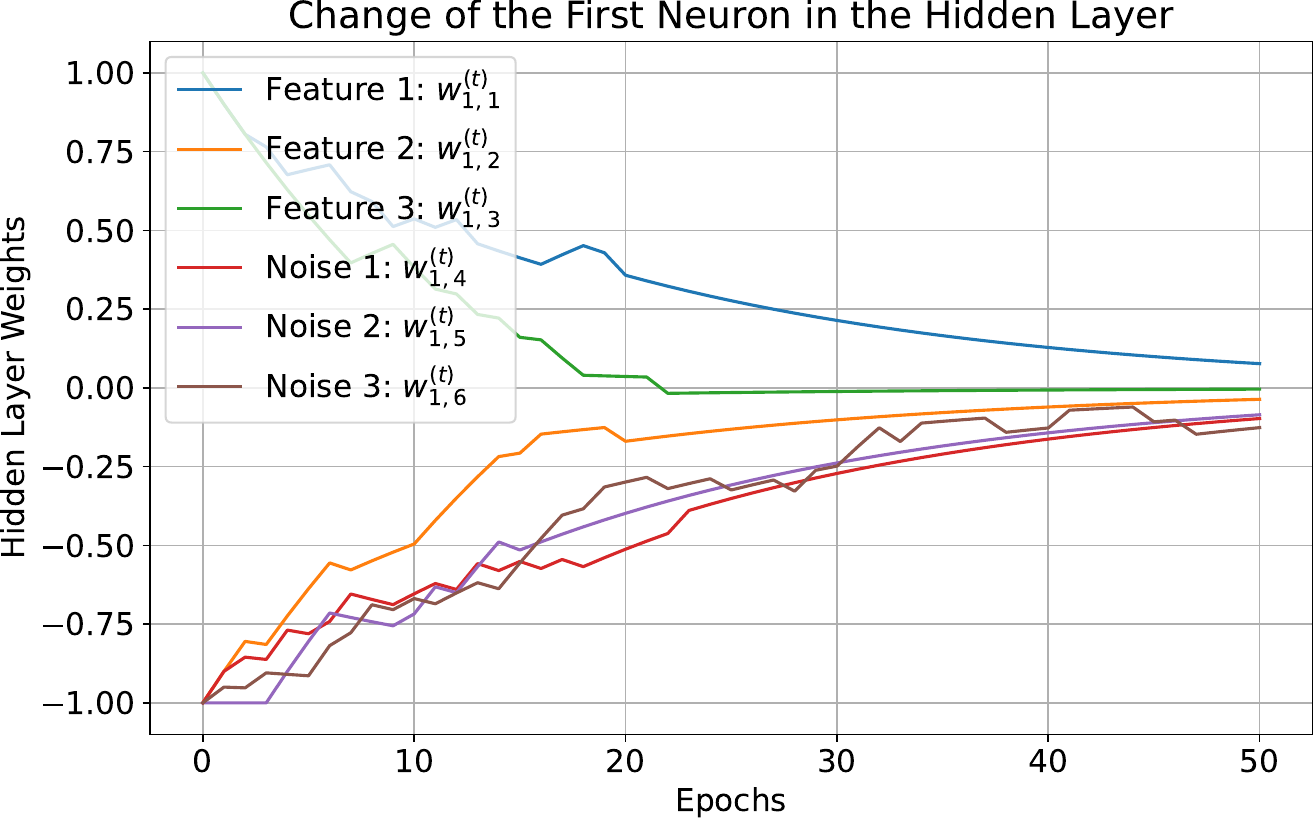}
  \captionof{figure}{Illustration of a $3$-parity $\bad$ neuron with initial weights $w_{1,1}^{(0)} = 1$, $w_{1,2}^{(0)} = -1$, $w_{1,3}^{(0)} = 1$, and $a_{1} = 1$.}
  \label{fig:bad3}
\end{minipage}
\end{figure}

\begin{figure}[ht!]
\centering
\begin{minipage}{.48\textwidth}
  \centering
  \includegraphics[width=1\linewidth]{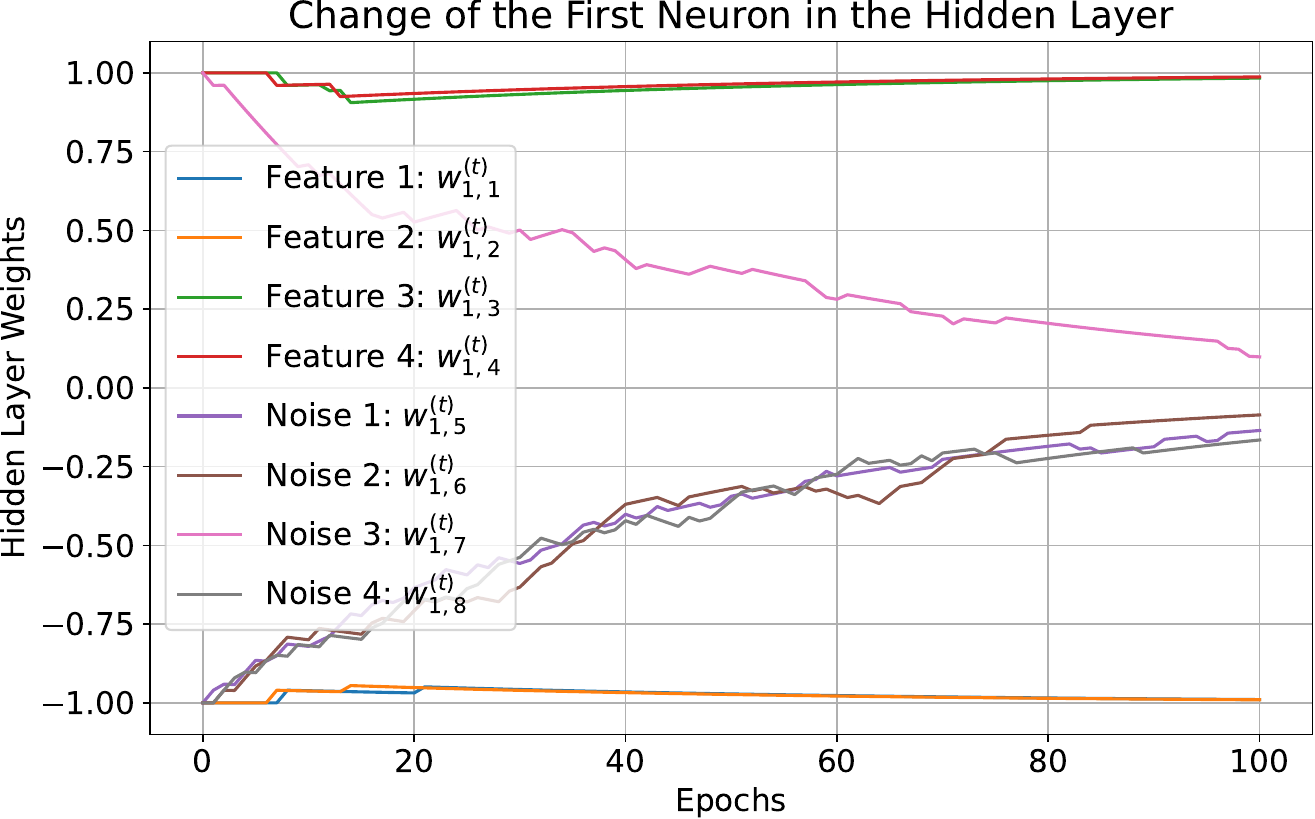}
  \captionof{figure}{Illustration of a $4$-parity $\good$ neuron with initial weights $w_{1,1}^{(0)} = -1$, $w_{1,2}^{(0)} = -1$, $w_{1,3}^{(0)} = 1$, $w_{1,4}^{(0)} = 1$, and $a_{1} = 1$.}
  \label{fig:good4}
\end{minipage}%
\hspace{0.04\textwidth}%
\begin{minipage}{.48\textwidth}
  \centering
  \includegraphics[width=1\linewidth]{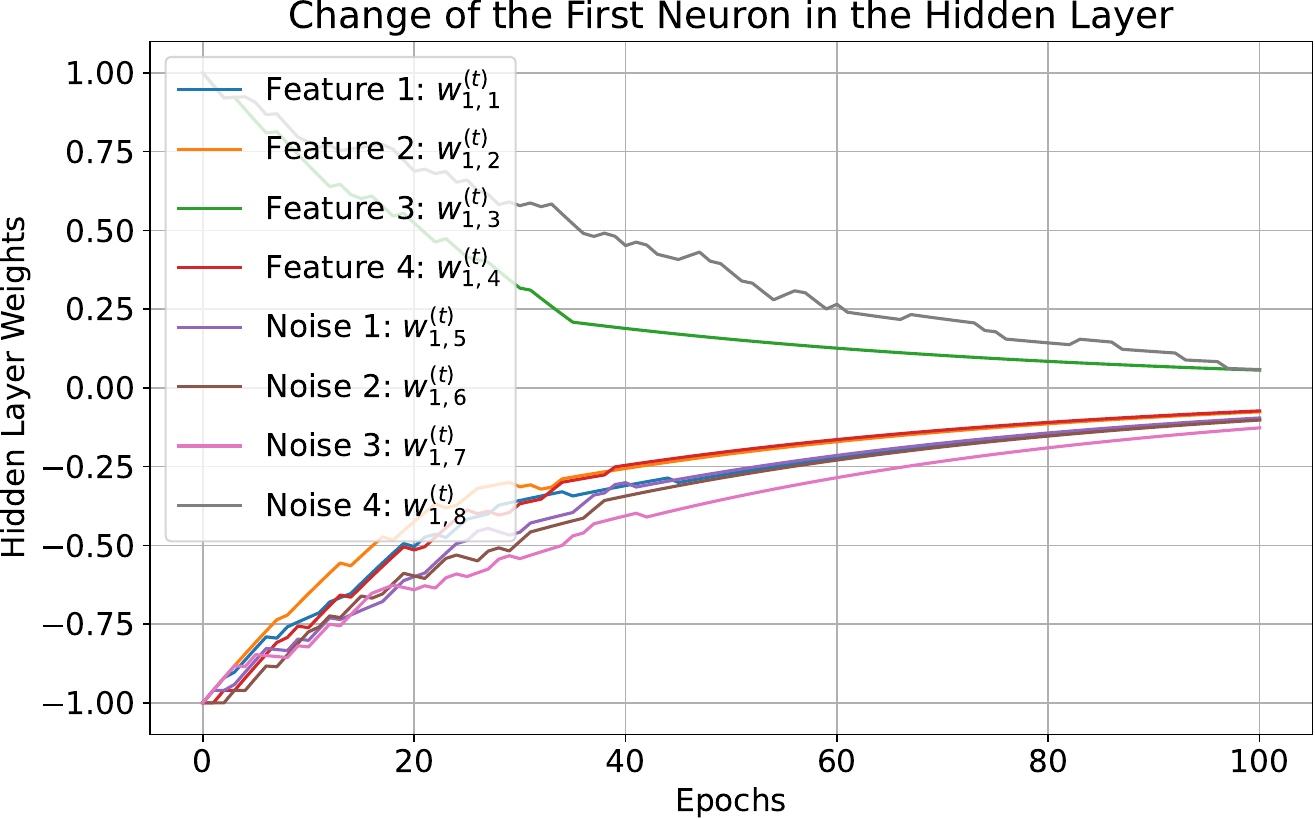}
  \captionof{figure}{Illustration of a $4$-parity $\bad$ neuron with initial weights $w_{1,1}^{(0)} = -1$, $w_{1,2}^{(0)} = -1$, $w_{1,3}^{(0)} = 1$, $w_{1,4}^{(0)} = -1$, and $a_{1} = -1$.}
  \label{fig:bad4}
\end{minipage}
\end{figure}


\section{Preliminary Lemmas}
During the initialization phase of a neural network, neurons can be categorized into $2^{k}$ distinct groups. This classification is based on whether each feature coordinate is positive or negative. We define these groups as follows:
\begin{align*}
    \Omega_{b_1b_2\cdots b_k} &= \{r \in [m] \,|\, \sign(w_{r,j}^{(0)})=b_j, \forall j \in [k] \},
\end{align*}
where $b_1,b_2,\cdots,b_k\in\{\pm1\}$. To illustrate with specific examples, consider the following special cases:
\begin{align*}
    \Omega_{11\cdots1} &= \{r \in [m] \,|\, w_{r,j}^{(0)} > 0, \forall j \in [k]\},\\
    \Omega_{-1-1\cdots-1} &= \{r \in [m] \,|\, w_{r,j}^{(0)} < 0, \forall j \in [k]\}.
\end{align*}
In these cases, $\Omega_{11\cdots1}$ represents the group of neurons where all initial weights are positive across the $k$ features, while $\Omega_{-1-1\cdots-1}$ consists of neurons with all initial weights being negative. Within each group of neurons, we can further subdivide them into two subgroups based on the value of $a_r$. Let's denote
\begin{align*}
    \Omega_{\mathrm{g}}=\bigg\{r\in[m]\,\bigg|\,a_r=\prod_{j=1}^{k}\sign(w_{r,j}^{(0)})\bigg\}, \Omega_{\mathrm{b}}=\bigg\{r\in[m]\,\bigg|\,a_r=-\prod_{j=1}^{k}\sign(w_{r,j}^{(0)})\bigg\},
\end{align*}
where $\Omega_g$ denotes the $\good$ neuron set and $\Omega_b$ denotes the $\bad$ neuron set. We will later demonstrate in the proof and experiments that neurons in $\Omega_{\mathrm{g}}$ and neurons in $\Omega_{\mathrm{b}}$ exhibit distinct behaviors during the training process. Specifically, for neurons in $\Omega_{\mathrm{g}}$, the feature coordinates will remain largely unchanged from their initial values throughout training, while the noise coordinates will decrease to a lower order compared to the feature coordinates. On the other hand, for neurons in $\Omega_{\mathrm{b}}$, both feature coordinates and noise coordinates will decrease to a lower order compared to their initial values. 

In order to establish the test error result, it is essential to impose a condition on the initialization. Specifically, the number of neurons for each type of initialization should be approximately equal.
\begin{lemma}\label{lm:set size}
With probability at least $1-\delta$ for the randomness in the neural network's initialization, the sizes of the sets $\Omega_{\mathrm{g}}$ and $\Omega_{\mathrm{b}}$ are bounded as follows:
\begin{align*}
    |\Omega_{\mathrm{g}}|,|\Omega_{\mathrm{b}}|\in[(1-\alpha)m/2,(1+\alpha)m/2].
\end{align*}
Besides, the intersections of $\Omega_{b_1b_2\cdots b_k}$ with both $\Omega_{\mathrm{g}}$ and $\Omega_{\mathrm{b}}$ are also bounded within a specified range:
\begin{equation*}
    |\Omega_{b_1b_2\cdots b_k}\cap\Omega_{\mathrm{g}}|,|\Omega_{b_1b_2\cdots b_k}\cap\Omega_{\mathrm{b}}|\in[(1-\alpha)m/2^{k+1},(1+\alpha)m/2^{k+1}],
\end{equation*}
where
\begin{equation*}
    \alpha=\sqrt{\frac{3\cdot 2^{k+1}\log(2^{k+2}/\delta)}{m}}.
\end{equation*}
\end{lemma}
\begin{proof}
Let $X_r=\ind\big[w_{r,1}^{(0)}>0,\cdots,w_{r,k}^{(0)}>0,a_r=1\big]$. Then, by Chernoff bound, we have
\begin{equation*}
    \PP\Bigg(\bigg|\sum_{r=1}^{m}X_{r}-\frac{m}{2^{k+1}}\bigg|\geq\alpha\cdot\frac{m}{2^{k+1}}\Bigg)\leq2\exp\Big(-\frac{\alpha^{2}m}{3\cdot2^{k+1}}\Big).
\end{equation*}
Then, with probability at least $1-\delta$, we have
\begin{equation*}
(1-\alpha)\cdot\frac{m}{2^{k+1}}\leq|\Omega_{11\cdots1}\cap\Omega_{\mathrm{g}}|\leq(1+\alpha)\cdot\frac{m}{2^{k+1}},
\end{equation*}
where
\begin{equation*}
    \alpha=\sqrt{\frac{3\cdot 2^{k+1}\log(2/\delta)}{m}}.
\end{equation*}
By applying union bound to all $2^{k+1}$ kinds of initialization, with probability at least $1-\delta$ it holds that for any $\Omega_{b_1b_2\cdots b_k}\cap\Omega_{\mathrm{g}}$ and $\Omega_{b_1b_2\cdots b_k}\cap\Omega_{\mathrm{b}}\big((b_1,\cdots,b_k)\in\{\pm1\}^{k}\big)$
\begin{align*}
    (1-\alpha)\cdot\frac{m}{2^{k+1}}&\leq|\Omega_{b_1b_2\cdots b_k}\cap\Omega_{\mathrm{g}}|\leq(1+\alpha)\cdot\frac{m}{2^{k+1}},\\
    (1-\alpha)\cdot\frac{m}{2^{k+1}}&\leq|\Omega_{b_1b_2\cdots b_k}\cap\Omega_{\mathrm{b}}|\leq(1+\alpha)\cdot\frac{m}{2^{k+1}},
\end{align*}
where
\begin{equation*}
    \alpha=\sqrt{\frac{3\cdot 2^{k+1}\log(2^{k+2}/\delta)}{m}}.
\end{equation*}
\end{proof}
\section{Warmup: Population Sign GD}\label{sec:warmup}
In this section, we train a neural network with gradient descent on the distribution $\cD$. Here we use correlation loss function $\ell(y,\hat{y})=1-y\hat{y}$. Then, the loss on this attribution is $L_{\cD}(\Wb)=\EE_{(\xb,y)\sim\cD}[\ell(y,f(\Wb,\xb))]$, and we perform the following updates:
\begin{align*}
    \wb_{r}^{(t+1)}&=(1-\eta\lambda)\wb_{r}^{(t)}-\eta\cdot\tilde{\sign}\big(\nabla_{\wb_r}L_{\cD}(\Wb^{(t)})\big).
\end{align*}
We assume the network is initialized with a symmetric initialization: for every $r\in[m]$, initialize $\wb_{r}^{(0)}\sim\text{Unif}(\{-1,1\}^{d})$ and initialize $a_{r}\sim\text{Unif}(\{-1,1\})$. 

\begin{lemma}\label{lm:exact gradient}
The following coordinate-wise population sign gradient update rules hold:
\begin{align*}
    w_{r,j}^{(t+1)}&=(1-\eta\lambda)w_{r,j}^{(t)}+\eta\cdot\Tilde{\sign}\Bigg(k!a_{r}\cdot\frac{w_{r,1}^{(t)}w_{r,2}^{(t)}\cdots w_{r,k}^{(t)}}{w_{r,j}^{(t)}}\Bigg), &&j\in[k],\\
    w_{r,j}^{(t+1)}&=(1-\eta\lambda)w_{r,j}^{(t)}, &&j\notin[k].
\end{align*}
\end{lemma}
\begin{proof}
For population gradient descent, we have
\begin{align*}
    \wb_{r}^{(t+1)}&=(1-\eta\lambda)\cdot\wb_{r}^{(t)}-\eta\cdot\Tilde{\sign}\big(\nabla_{\wb_r}L_{\cD}(\Wb)\big)\\
    &=(1-\eta\lambda)\cdot\wb_{r}^{(t)}-\eta\cdot\Tilde{\sign}\big(a_r\cdot\EE_{(\xb,y)\sim\cD}[\nabla_{\wb_r}\sigma(\la\wb_{r}^{(t)},\xb\ra)]\big)\\
    &=(1-\eta\lambda)\cdot\wb_{r}^{(t)}-\eta\cdot\Tilde{\sign}\big(k a_r\cdot\EE_{(\xb,y)\sim\cD}[y\xb(\la\wb_{r}^{(t)},\xb\ra)^{k-1}]\big).
\end{align*}
Notice that for $j\notin[k]$, we have
\begin{align*}
    \EE_{(\xb,y)\sim \cD}[y x_j(\la\wb_{r}^{(t)},\xb\ra)^{k-1}]&=\EE_{(\xb,y)\sim \cD}\bigg[x_1x_2\cdots x_k x_j\bigg(\sum_{j_1,\cdots,j_{k-1}}w_{r,j_1}^{(t)}\cdots w_{r,j_{k-1}}^{(t)}x_{j_1}\cdots x_{j_{k-1}}\bigg)\bigg]\\
    &=\sum_{j_1,\cdots,j_{k-1}}w_{r,j_1}^{(t)}\cdots w_{r,j_{k-1}}^{(t)}\cdot\EE_{(\xb,y)\sim \cD}\big[x_1x_2\cdots x_k x_j x_{j_1}\cdots x_{j_{k-1}}\big]\\
    &=0,
\end{align*}
where the last equality is because $\{j_1,\cdots,j_{k-1}\}\subsetneq\{1,\cdots,k,j\}$. This implies that
\begin{align*}
    w_{r,j}^{(t+1)}&=(1-\eta\lambda)\cdot w_{r,j}^{(t)}.
\end{align*}
For $j\in[k]$, we have
\begin{align*}
    \EE_{(\xb,y)\sim \cD}[y x_j(\la\wb_{r}^{(t)},\xb\ra)^{k-1}]&=\EE_{(\xb,y)\sim \cD}\bigg[x_1x_2\cdots x_k x_j\bigg(\sum_{j_1,\cdots,j_{k-1}}w_{r,j_1}^{(t)}\cdots w_{r,j_{k-1}}^{(t)}x_{j_1}\cdots x_{j_{k-1}}\bigg)\bigg]\\
    &=\sum_{j_1,\cdots,j_{k-1}}w_{r,j_1}^{(t)}\cdots w_{r,j_{k-1}}^{(t)}\cdot\EE_{(\xb,y)\sim \cD}[x_1x_2\cdots x_k x_j x_{j_1}\cdots x_{j_{k-1}}]\\
    &=(k-1)!\cdot\frac{w_{r,1}^{(t)}w_{r,2}^{(t)}\cdots w_{r,k}^{(t)}}{w_{r,j}^{(t)}},
\end{align*}
where the last inequality is because $\EE_{(\xb,y)\sim \cD}[x_1x_2\cdots x_k x_j x_{j_1}\cdots x_{j_{k-1}}]\neq 0$ if and only if $\{j,j_1,\cdots,j_{k-1}\}=\{1,2,\cdots,k\}$. It follows that
\begin{align*}
    w_{r,j}^{(t+1)}&=(1-\eta\lambda)\cdot w_{r,j}^{(t)}+\eta\cdot\Tilde{\sign}\Bigg(k!a_r\cdot\frac{w_{r,1}^{(t)}w_{r,2}^{(t)}\cdots w_{r,k}^{(t)}}{w_{r,j}^{(t)}}\Bigg).
\end{align*}
\end{proof}

Given the update rules in Lemma~\ref{lm:exact gradient}, we observe distinct behaviors for good neurons ($\Omega_{\mathrm{g}}$) and bad neurons ($\Omega_{\mathrm{b}}$). The following corollary illustrates these differences:
\begin{corollary}\label{cor:exact gd}
For any neuron $r\in\Omega_{b_1\cdots b_k}\cap\Omega_{\mathrm{g}}$, the update rule for feature coordinates $(j\in[k])$ is given by:
\begin{align*}
    b_{j}w_{r,j}^{(t+1)}&=(1-\eta\lambda)b_{j}w_{r,j}^{(t)}\\
    &\qquad+\eta\cdot\frac{\sign(b_1 w_{r,1}^{(t)})\sign(b_2 w_{r,2}^{(t)})\cdots\sign(b_k w_{r,k}^{(t)})}{\sign(b_j w_{r,j}^{(t)})}\cdot\ind\Bigg[\frac{|w_{r,1}^{(t)}w_{r,2}^{(t)}\cdots w_{r,k}^{(t)}|}{|w_{r,j}^{(t)}|}\geq\frac{\rho}{k!}\Bigg].
\end{align*}
For any neuron $r\in\Omega_{b_1\cdots b_k}\cap\Omega_{\mathrm{b}}$, the update rule for feature coordinates $(j\in[k])$ is given by:
\begin{align*}
    b_{j}w_{r,j}^{(t+1)}&=(1-\eta\lambda)b_{j}w_{r,j}^{(t)}\\
    &\qquad-\eta\cdot\frac{\sign(b_1 w_{r,1}^{(t)})\sign(b_2 w_{r,2}^{(t)})\cdots\sign(b_k w_{r,k}^{(t)})}{\sign(b_j w_{r,j}^{(t)})}\cdot\ind\Bigg[\frac{|w_{r,1}^{(t)}w_{r,2}^{(t)}\cdots w_{r,k}^{(t)}|}{|w_{r,j}^{(t)}|}\geq\frac{\rho}{k!}\Bigg].
\end{align*}
\end{corollary}
By setting the regularization parameter $\lambda$ to $1$, we observe a noteworthy property of the weight of feature coordinates in good neurons. This is formalized in the following lemma:
\begin{lemma}\label{lm:exact gd feature}
Assume $\lambda = 1$ and $\rho < k!$. For a neuron $r\in\Omega_{b_1\cdots b_k}\cap\Omega_{\mathrm{g}}$, the weight associated with any feature coordinate remains constant across all time steps $t\geq0$. Specifically, it holds that: 
\begin{align*}
    b_j w_{r,j}^{(t)}=1,\quad \forall j \in [k], \, t \geq 0.
\end{align*}
\end{lemma}
\begin{proof}
We prove this by using induction. The result is obvious at $t=0$. Suppose the result holds when $t=\tilde{t}$. Then, according to Corollary~\ref{cor:exact gd}, we have
\begin{align*}
    b_j w_{r,j}^{(\tilde{t}+1)}&=(1-\eta\lambda)b_j w_{r,j}^{(\tilde{t})}\\
    &\qquad+\eta\cdot\frac{\sign(b_{1}w_{r,1}^{(\tilde{t})})\sign(b_{2}w_{r,2}^{(\tilde{t})})\cdots \sign(b_{k}w_{r,k}^{(\tilde{t})})}{\sign(b_{j}w_{r,j}^{(\tilde{t})})}\cdot\ind\Bigg[\frac{|w_{r,1}^{(\tilde{t})}w_{r,2}^{(\tilde{t})}\cdots w_{r,k}^{(\tilde{t})}|}{|w_{r,j}^{(t)}|}\geq\frac{\rho}{k!}\Bigg]\\
    &=(1-\eta\lambda)+\eta\cdot\ind[k!\geq\rho]\\
    &=1.
\end{align*}
\end{proof}
The following lemma demonstrates that for $\bad$ neurons, the weights of feature coordinates tend to shrink over time. The dynamics of this shrinking are characterized as follows:
\begin{lemma}\label{lm:exact dud feature}
Assume $\lambda = 1$ and $\eta/(1-\eta\lambda)<(\rho/k!)^{\frac{1}{k-1}}$. For a neuron $r\in\Omega_{b_1 b_2\cdots b_k}\cap\Omega_{\mathrm{b}}$, the weights of any two feature coordinates $j$ and $j'$ (where $j,j'\in[k]$) are equal at any time step, that is, $b_{j}w_{r,j}^{(t)}=b_{j'}w_{r,j'}^{(t)}$. Furthermore, the weight of any feature coordinate $j\in[k]$ evolves according to the following inequality for any $t\geq0$:
\begin{equation*}
    0<b_{j}w_{r,j}^{(t+1)}\leq(1-\eta\lambda)b_{j}w_{r,j}^{(t)}.
\end{equation*}
\end{lemma}
\begin{proof}
We prove this by using induction. We prove the following three hypotheses:
\begin{align}
    &b_{j}w_{r,j}^{(t)}=b_{j'}w_{r,j'}^{(t)},&&\forall j,j'\in[k].\tag{H$_1$}\label{b4:h1}\\
    &b_{j}w_{r,j}^{(t)}>0,&&\forall j\in[k].\tag{H$_2$}\label{b4:h2}\\
    &b_{j}w_{r,j}^{(t+1)}\leq(1-\eta\lambda)b_{j}w_{r,j}^{(t)},&&\forall j\in[k].\tag{H$_3$}\label{b4:h3}
\end{align}
We will show that \ref{b4:h1}$(0)$ and \ref{b4:h2}$(0)$ are true and that for any $t\geq0$ we have
\begin{itemize}[leftmargin=*,nosep]
    \item \ref{b4:h2}$(t)\Longrightarrow$\ref{b4:h3}$(t)$,
    \item \ref{b4:h1}$(t)$, \ref{b4:h2}$(t)\Longrightarrow$\ref{b4:h1}$(t+1)$,
    \item \ref{b4:h1}$(t)$, \ref{b4:h2}$(t)\Longrightarrow$\ref{b4:h2}$(t+1)$.
\end{itemize}
\ref{b4:h1}$(0)$ and \ref{b4:h2}$(0)$ are obviously true since $b_{j}w_{r,j}^{(0)}=1$ for any $j\in[k]$. Next, we prove that \ref{b4:h2}$(t)\Longrightarrow$\ref{b4:h3}$(t)$ and \ref{b4:h1}$(t)$, \ref{b4:h2}$(t)\Longrightarrow$\ref{b4:h1}$(t+1)$. According to Corollary~\ref{cor:exact gd}, we have
\begin{align}
    b_{j}w_{r,j}^{(t+1)}&=(1-\eta\lambda)b_{j}w_{r,j}^{(t)} \notag\\
    &\qquad-\eta\cdot\frac{\sign(b_{1}w_{r,1}^{(t)})\sign(b_{2}w_{r,2}^{(t)})\cdots \sign(b_{k}w_{r,k}^{(t)})}{\sign(b_{j}w_{r,j}^{(t)})}\cdot\ind\Bigg[\frac{|w_{r,1}^{(t)}w_{r,2}^{(t)}\cdots w_{r,k}^{(t)}|}{|w_{r,j}^{(t)}|}\geq\frac{\rho}{k!}\Bigg]\notag\\
    &=(1-\eta\lambda)b_{j}w_{r,j}^{(t)}-\eta\cdot\ind\Bigg[\frac{|w_{r,1}^{(t)}w_{r,2}^{(t)}\cdots w_{r,k}^{(t)}|}{|w_{r,j}^{(t)}|}\geq\frac{\rho}{k!}\Bigg]\label{eq:simplified rule}\\
    &\leq(1-\eta\lambda)b_{j}w_{r,j}^{(t)},\notag
\end{align}
where the second equality is by \ref{b4:h2}$(t)$. This verifies \ref{b4:h3}$(t)$. Besides, given \ref{b4:h1}$(t)$, we have
\begin{align*}
    \frac{|w_{r,1}^{(t)}w_{r,2}^{(t)}\cdots w_{r,k}^{(t)}|}{|w_{r,j}^{(t)}|}=\frac{|w_{r,1}^{(t)}w_{r,2}^{(t)}\cdots w_{r,k}^{(t)}|}{|w_{r,j'}^{(t)}|}, &&\forall j,j'\in[k].
\end{align*}
Plugging this into \eqref{eq:simplified rule}, we can get
\begin{align*}
    b_{j}w_{r,j}^{(t+1)}=b_{j'}w_{r,j'}^{(t+1)},&&\forall j,j'\in[k],
\end{align*}
which verifies \ref{b4:h1}$(t+1)$. Finally, we prove that \ref{b4:h1}$(t)$, \ref{b4:h2}$(t)\Longrightarrow$ \ref{b4:h2}$(t+1)$. By \eqref{eq:simplified rule} and \ref{b4:h1}$(t)$, we have
\begin{align*}
    b_{j}w_{r,j}^{(t+1)}&=(1-\eta\lambda)b_{j}w_{r,j}^{(t)}-\eta\cdot\ind\Big[|w_{r,j}^{(t)}|^{k-1}\geq\frac{\rho}{k!}\Big].
\end{align*}
If $|w_{r,j}^{(t)}|<(\rho/k!)^{\frac{1}{k-1}}$, we can get
\begin{align*}
    b_{j}w_{r,j}^{(t+1)}=(1-\eta\lambda)b_{j}w_{r,j}^{(t)}>0.
\end{align*}
If $|w_{r,j}^{(t)}|\geq(\rho/k!)^{\frac{1}{k-1}}$, given that
\begin{align*}
    \frac{\eta}{1-\eta\lambda}<(\rho/k!)^{\frac{1}{k-1}},
\end{align*}
we can get
\begin{align*}
    b_{j}w_{r,j}^{(t+1)}=(1-\eta\lambda)b_{j}w_{r,j}^{(t)}-\eta>0.
\end{align*}
\end{proof}
Given Lemma~\ref{lm:exact gd feature} and Lemma~\ref{lm:exact dud feature}, we can directly get the change of neurons of all kinds of initialization.
\begin{corollary}\label{cor:all neuron}
Assume $\lambda = 1$, $\rho < k!$ and $\eta/(1-\eta\lambda)<(\rho/k!)^{\frac{1}{k-1}}$. For any fixed $(b_1,b_2,\cdots,b_k)\in\{\pm1\}^{k}$, considering $r\in\Omega_{b_1b_2\cdots b_k}$, we have the following statements hold. 
\begin{itemize}[leftmargin=*]
    \item For any neuron $r\in\Omega_{b_1b_2\cdots b_k}\cap\Omega_{\mathrm{g}}$, the weights of feature coordinates remain the same as initialization: $w_{r,j}^{(t)}=b_j$ for any $t\geq0$ and $j\in[k]$.
    \item For any neuron $r\in\Omega_{b_1b_2\cdots b_k}\cap\Omega_{\mathrm{b}}$, the weights of feature coordinates will shrink simultaneously over time: $b_{j}w_{r,j}^{(t)}=b_{j'}w_{r,j'}^{(t)}$ for any $t\geq0$ and $j,j'\in[k]$ and
\begin{equation*}
    0<b_jw_{r,j}^{(t+1)}\leq(1-\eta\lambda)\cdot b_jw_{r,j}^{(t)},
\end{equation*}
for any $t\geq0$ and $j\in[k]$.
\end{itemize}
\end{corollary}
Building on Corollary~\ref{cor:all neuron}, we can now characterize the trajectory of all neurons over time. Specifically, after a time period $T=\Theta(k\eta^{-1}\lambda^{-1}\log(d))$, the following observations about neuron weights hold:
\begin{lemma}\label{lm:exact order}
Assume $\lambda = 1$, $\rho < k!$ and $\eta/(1-\eta\lambda)<(\rho/k!)^{\frac{1}{k-1}}$. For $T\geq (k+1)\eta^{-1}\lambda^{-1}\log d$, it holds that
\begin{align*}
    w_{r,j}^{(T)}&=b_j,&&\forall r\in\Omega_{b_1b_1\cdots b_k}\cap\Omega_{\mathrm{g}},j\in[k],\\
    |w_{r,j}^{(T)}|&\leq d^{-(k+1)}, &&\forall r\in\Omega_{\mathrm{g}},j\in[d]\setminus[k],\\
    |w_{r,j}^{(T)}|&\leq d^{-(k+1)}, &&\forall r\in\Omega_{\mathrm{b}},j\in[d].
\end{align*}
\end{lemma}
\begin{proof}
The first equality is obvious according to Corollary~\ref{cor:all neuron}. We only need to prove the inequalities. According to Lemma~\ref{lm:exact gradient}, we have for any $r\in[m]$ and $j\in[d]\setminus[k]$ that
\begin{equation*}
    |w_{r,j}^{(T)}|=(1-\eta\lambda)^{T}|w_{r,j}^{(0)}|=\big((1-\eta\lambda)^{(\eta\lambda)^{-1}}\big)^{T\eta\lambda}\leq\exp(-T\eta\lambda)\leq d^{-(k+1)},
\end{equation*}
where the last inequality is by $T\geq k\eta^{-1}\lambda^{-1}\log(d)$.
According to Corollary~\ref{cor:all neuron}, for any $r\in\Omega_{\mathrm{b}}$ and $j\in[k]$, we have that
\begin{equation*}
    |w_{r,j}^{(T)}|\leq(1-\eta\lambda)^{T}|w_{r,j}^{(0)}|=\big((1-\eta\lambda)^{(\eta\lambda)^{-1}}\big)^{T\eta\lambda}\leq\exp(-T\eta\lambda)\leq d^{-(k+1)}.
\end{equation*}
\end{proof}

\begin{lemma}\label{lm:nn approx}
Under Condition~\ref{Cond}, with a probability of at least $1-\delta$ with respect to the randomness in the neural network's initialization, trained neural network $f(\Wb^{(T)},\xb)$ approximates accurate classifier $(m/2^{k+1})\cdot f(\Wb^{*},\xb)$ well: 
\begin{equation*}
    \PP_{\xb\sim\cD_{\xb}}\bigg(\frac{f(\Wb^{(T)},\xb)}{(m/2^{k+1})\cdot f(\Wb^{*},\xb)}\in[0.5,1.5]\bigg)\geq 1-\epsilon.
\end{equation*}
\end{lemma}

\begin{proof}
First, we can rewrite \eqref{eq:optimal} as follows:
\begin{align*}
    f(\Wb^{*},\xb)&=\sum_{(b_1,\cdots,b_k)\in\{\pm1\}^{k}}(b_1\cdots b_k)\cdot\sigma(\la\wb_{b_1\cdots b_k}^{*},\xb\ra),
\end{align*}
where $\wb_{b_1\cdots b_k}^{*}=[b_1,b_2,\cdots,b_k,0,\cdots,0]^{\top}$. To prove this lemma, we need to estimate the noise part of the inner product $\la\wb_{r}^{(T)},\xb\ra$. By Hoeffding's inequality, we have the following upper bound for the noise part $\sum_{j=k+1}^{d}w_{r,j}^{(T)}x_j$:
\begin{align*}
    \PP_{\xb}\Bigg(\Bigg|\sum_{j=k+1}^{d}w_{r,j}^{(T)}x_{j}-\EE\Bigg[\sum_{j=k+1}^{d}w_{r,j}^{(T)}x_{j}\Bigg]\Bigg|\geq x\Bigg)&=\PP_{\xb}\Bigg(\Bigg|\sum_{j=k+1}^{d}w_{r,j}^{(t)}x_{j}\Bigg|\geq x\Bigg)\\
    &\leq2\exp\bigg(-\frac{x^2}{2\sum_{j=k+1}^{d}[w_{r,j}^{(t)}]^2}\bigg).
\end{align*}
Then with probability at least $1-\epsilon/m$ we have that for fixed $r\in[m]$
\begin{align*}
\Bigg|\sum_{j=k+1}^{d}w_{r,j}^{(T)}x_{j}\Bigg| &\leq \sqrt{2}\log(2m/\epsilon)\sqrt{\sum_{j=k+1}^{d}[w_{r,j}^{(T)}]^2}\\
&=\sqrt{2}\log(2m/\epsilon)\|\wb_{r,[k+1:d]}^{(T)}\|_{2}\\
&\leq\sqrt{2}\log(2m/\epsilon)d^{-(k+1)}(d-k)^{1/2}\\
&\leq d^{-k},
\end{align*}
where the last inequality is by Condition~\ref{Cond}. By applying union bound to all $m$ neurons, with probability at least $1-\epsilon$ we have that for any $r\in[m]$
\begin{align}
    \bigg|\sum_{j=k+1}^{d}w_{r,j}^{(T)}x_{j}\bigg|\leq d^{-k}. \label{ineq:noise part}
\end{align}
By Lemma~\ref{lm:exact order}, we have
\begin{align*}
    &\Big|f(\Wb^{(T)},\xb)-\frac{m}{2^{k+1}}\cdot f(\Wb^{*},\xb)\Big|\\
    &\leq\Bigg|\sum_{r\in\Omega_{\mathrm{g}}}a_r\sigma(\la\wb_{r}^{(T)},\xb\ra)-\frac{m}{2^{k+1}}\cdot\sum_{r=1}^{2^{k}}a_{r}^{*}\sigma(\la\wb_{r}^{*},\xb\ra)\Bigg|+\Bigg|\sum_{r\in\Omega_{\mathrm{b}}}a_r\sigma(\la\wb_{r}^{(T)},\xb\ra)\Bigg|\\
    &\leq\sum_{(b_1,\cdots,b_k)\in\{\pm1\}^{k}}\Bigg|\sum_{\Omega_{b_1\cdots b_k}\cap\Omega_{\mathrm{g}}}\sigma(\la\wb_{r}^{(T)},\xb\ra)-\frac{m}{2^{k+1}}\cdot\sigma(\la\wb_{b_1\cdots b_k}^{*},\xb\ra)\Bigg|+\Bigg|\sum_{r\in\Omega_{\mathrm{b}}}a_r\sigma(\la\wb_{r}^{(T)},\xb\ra)\Bigg|\\
    &\leq\sum_{(b_1,\cdots,b_k)\in\{\pm1\}^{k}}\Bigg(\Big(\frac{\alpha m}{2^{k+1}}\Big)\bigg|\sum_{j=1}^{k}b_j x_j\bigg|^{k}+\sum_{\Omega_{b_1\cdots b_k}\cap\Omega_{\mathrm{g}}}\big|\sigma(\la\wb_{r}^{(T)},\xb\ra)-\sigma(\la\wb_{b_1\cdots b_k}^{*},\xb\ra)\big|\Bigg)\\
    &\qquad+\Bigg|\sum_{r\in\Omega_{\mathrm{b}}}a_r\sigma(\la\wb_{r}^{(T)},\xb\ra)\Bigg|\\
    &\leq\sum_{(b_1,\cdots,b_k)\in\{\pm1\}^{k}}\Bigg(\Big(\frac{\alpha m}{2^{k+1}}\Big)\bigg|\sum_{j=1}^{k}b_j x_j\bigg|^{k}+\sum_{\Omega_{b_1\cdots b_k}\cap\Omega_{\mathrm{g}}}kd^{-k}\big(k+d^{-k}\big)^{k-1}\Bigg)+|\Omega_{\mathrm{b}}|(2d^{-k}\big)^{k}\\
    &\leq\Big(\frac{\alpha m}{2^{k+1}}\Big)2k^{k}(1+e^{-2})^{k}+|\Omega_{\mathrm{g}}|kd^{-k}\big(k+d^{-k}\big)^{k-1}+|\Omega_{\mathrm{b}}| \big(2d^{-k}\big)^{k}\\
    &\leq \Big(\alpha k^k 2^{-k}(1+e^{-2})^{k}+0.5(1+\alpha)kd^{-k}\big(k+d^{-k}\big)^{k-1}+0.5(1+\alpha)\big(2d^{-k}\big)^{k}\Big)\cdot m,
\end{align*}
where the first three inequalities are by triangle inequality and Lemma~\ref{lm:set size}; the fourth inequality is due to
\begin{align*}
    &\big|\sigma(\la\wb_{r}^{(T)},\xb\ra)-\sigma(\la\wb_{b_1\cdots b_k}^{*},\xb\ra)\big|\\
    &\leq\big|(\la\wb_{r}^{(T)},\xb\ra)^{k}-(\la\wb_{b_1\cdots b_k}^{*},\xb\ra)^{k}\big|\\
    &\leq|\la\wb_{r}^{(T)},\xb\ra-\la\wb_{b_1\cdots b_k}^{*},\xb\ra|\cdot k(\max\{|\la\wb_{r}^{(T)},\xb\ra|,|\la\wb_{b_1\cdots b_k}^{*},\xb\ra|\})^{k-1}\\
    &=\Bigg|\sum_{j=k+1}^{d}w_{r,j}^{(T)}x_j\Bigg|\cdot k\Bigg(\max\Bigg\{\Bigg|\sum_{j=1}^{k}b_j x_j+\sum_{j=k+1}^{d}w_{r,j}^{(T)}x_j\Bigg|,\Bigg|\sum_{j=1}^{k}b_j x_j\Bigg|\Bigg\}\Bigg)^{k-1}\\
    &\leq k\Bigg|\sum_{j=k+1}^{d}w_{r,j}^{(T)}x_j\Bigg|\cdot\Bigg(\Bigg|\sum_{j=1}^{k}b_j x_j\Bigg|+\Bigg|\sum_{j=k+1}^{d}w_{r,j}^{(T)}x_j\Bigg|\Bigg)^{k-1}\\
    &\leq kd^{-k}\big(k+d^{-k}\big)^{k-1}
\end{align*}
by \eqref{ineq:noise part}, mean value theorem and Lemma~\ref{lm:exact order} and for $r\in\Omega_{\mathrm{b}}$
\begin{align}
    |\sigma(\la\wb_{r}^{(T)},\xb\ra)|&\leq|\la\wb_{r}^{(T)},\xb\ra|^{k}\leq\big(kd^{-(k+1)}+d^{-k}\big)^{k}\leq\big(2d^{-k}\big)^{k}.\label{ineq:bad neuron upbound}
\end{align}
Since $(m/2^{k+1})\cdot |f(\Wb^{*},\xb)|=0.5k!\cdot m$, then as long as 
\begin{equation*}
    \alpha\leq\frac{\sqrt{2\pi k}}{8}\cdot\Big(\frac{e+e^{-1}}{2}\Big)^{-k},\text{ and }\alpha\leq 1,
\end{equation*}
we have
\begin{align*}
    &\frac{|f(\Wb^{(T)},\xb)-(m/2^{k+1})\cdot f(\Wb^{*},\xb)|}{(m/2^{k+1})\cdot |f(\Wb^{*},\xb)|}\\
    &\leq\frac{\alpha k^k 2^{-k}(1+e^{-2})^{k}+0.5(1+\alpha)kd^{-k}\big(k+d^{-k}\big)^{k-1}+0.5(1+\alpha)\big(2d^{-k}\big)^{k}}{0.5k!}\\
    &\leq\frac{2\alpha k^k 2^{-k}(1+e^{-2})^{k}+(1+\alpha)kd^{-k}\big(k+d^{-k}\big)^{k-1}+(1+\alpha)\big(2d^{-k}\big)^{k}}{\sqrt{2\pi k}(k/e)^{k}}\\
    &=\sqrt{\frac{2}{\pi k}}\alpha\Big(\frac{e+e^{-1}}{2}\Big)^{k}+\sqrt{\frac{2}{\pi k}}\cdot\bigg(1+\frac{d^{-k}}{k}\bigg)^{k-1}\cdot\Big(\frac{e}{d}\Big)^{k}+\sqrt{\frac{2}{\pi k}}\cdot\Big(\frac{2e}{kd^{k}}\Big)^{k}\\
    &\leq 0.5,
\end{align*}
where the second inequality is by Stirling's approximation. Then it follows that
\begin{equation*}
    \frac{f(\Wb^{(T)},\xb)}{(m/2^{k+1})\cdot f(\Wb^{*},\xb)}\in[0.5,1.5].
\end{equation*}
\end{proof}

\section{Stochastic Sign GD}
In this section, we consider stochastic sign gradient descent for learning $k$-parity function. The primary aim of this section is to demonstrate that the trajectory produced by SGD closely resembles that of population GD. To begin, let's recall the update rule of SGD:
\begin{align*}
    \wb_{r}^{(t+1)}&=(1-\lambda\eta)\wb_{ r}^{(t)}+\eta\cdot\Tilde{\sign}\bigg(\frac{1}{|S_t|}\sum_{(\xb,y)\in S_{t}}\sigma'(\la\wb_{r}^{(t)},\xb\ra)\cdot a_r y\xb\bigg),
\end{align*}
where $|S_t|=B$. Our initial step involves estimating the approximation error between the stochastic gradient and the population gradient, detailed within the lemma that follows.
\begin{lemma}\label{lm:gradient approx error}
With probability at least $1-\delta$ with respect to the randomness of online data selection, for all $t \leq T$, the following bound holds true for each neuron $r\in[m]$ and for each coordinate $j\in[d]$:
\begin{align}
    \Bigg|\frac{1}{|S_{t}|}\sum_{(\xb,y)\in S_{t}}\sigma'(\la\wb_{r}^{(t)},\xb\ra)\cdot a_r y\xb-\EE_{(\xb,y)}\big[\sigma'(\la\wb_{r}^{(t)},\xb\ra)\cdot a_{r}yx_{j}\big]\Bigg|\leq\epsilon_{1}\cdot\|\wb_r^{(t)}\|_{2}^{k-1},\label{ineq:approx1}
\end{align}
where $\epsilon_1$ is defined as \begin{align*}
    \epsilon_1&=\frac{2^{k/2}k(\log(16mdBT/\delta))^{(k-1)/2}\log(8mdT/\delta)}{\sqrt{B}}+\frac{k d^{(k-3)/2}\delta}{8mBT}.
\end{align*}
\end{lemma}
\begin{proof}
To prove \eqref{ineq:approx1}, let us introduce the following notations:
\begin{align*}
    g_{r,j}(\xb,y)&=\sigma'(\la\wb_{r}^{(t)},\xb\ra)\cdot a_{r}yx_{j}=k(\la\wb_{r}^{(t)},\xb\ra)^{k-1}\cdot a_{r}yx_{j},\\
    h_{r,j}(\xb,y)&=g_{r,j}(\xb,y)\cdot\ind\big[|\la\wb_{r}^{(t)},\xb\ra|\leq\gamma\big],
\end{align*}
where $g_{r,j}(\xb,y)$ represents the gradient at the point $(\xb,y)$, and $h_{r,j}(\xb,y)$ denotes the truncated version of $g_{r,j}(\xb,y)$, which is employed for the convenience of applying the Hoeffding's inequality. Firstly, utilizing Hoeffding's inequality, we can assert the following:
\begin{equation}\label{ineq:truncated hoeffding}
    \PP\Bigg(\Bigg|\frac{1}{B}\sum_{(\xb,y)\in S_t}h_{r,j}(\xb,y)-\EE_{(\xb,y)\sim\cD}h_{r,j}(\xb,y)\Bigg|\geq x\Bigg)\leq2\exp\Bigg(-\frac{Bx^2}{2(k \gamma^{k-1})^{2}}\Bigg).
\end{equation}
Furthermore, we can establish an upper bound for the difference between the expectations of $h_{r,j}(\xb,y)$ and $g_{r,j}(\xb,y)$:
\begin{equation}\label{ineq:expectation diff}
\begin{aligned}
    \Big|\EE_{(\xb,y)\sim\cD}h_{r,j}(\xb,y)-\EE_{(\xb,y)\sim\cD}g_{r,j}(\xb,y)\Big|&=\Big|\EE_{(\xb,y)\sim\cD}g_{r,j}(\xb,y)\cdot\ind\big[|\la\wb_{r}^{(t)},\xb\ra|>\gamma\big]\Big|\\
    &\leq k d^{(k-1)/2}\|\wb_{r}^{(t)}\|_{2}^{k-1}\cdot\PP\big(|\la\wb_{r}^{(t)},\xb\ra|>\gamma\big)\\
    &\leq k d^{(k-1)/2}\|\wb_{r}^{(t)}\|_{2}^{k-1}\cdot 2\exp\bigg(-\frac{\gamma^{2}}{2\|\wb_{r}^{(t)}\|_{2}^{2}}\bigg),
\end{aligned}
\end{equation}
where the first inequality is by Cauchy inequality and the second inequality is by Hoeffding's inequality. Additionally, with high probability, the gradient and the truncated gradient are identical:
\begin{equation}\label{ineq:truncate diff}
\begin{aligned}
    &\PP\Bigg(\Bigg|\frac{1}{B}\sum_{(\xb,y)\in S_t}h_{r,j}(\xb,y)-\frac{1}{B}\sum_{(\xb,y)\in S_t}g_{r,j}(\xb,y)\Bigg|=0\Bigg)\\
    &=\PP\Bigg(\Bigg|\frac{1}{B}\sum_{(\xb,y)\in S_t}g_{r,j}(\xb,y)\ind\big[|\la\wb_{r}^{(t)},\xb\ra|>\gamma\big]\Bigg|=0\Bigg)\\
    &\geq\PP\Big(|\la\wb_{r}^{(t)},\xb\ra|\leq\gamma,\forall (\xb,y)\in S_t\Big)\\
    &\geq\Bigg(1-2\exp\bigg(-\frac{\gamma^{2}}{2\|\wb_{r}^{(t)}\|_{2}^{2}}\bigg)\Bigg)^{B}\\
    &\geq 1-2B\exp\bigg(-\frac{\gamma^{2}}{2\|\wb_{r}^{(t)}\|_{2}^{2}}\bigg),
\end{aligned}
\end{equation}
where the second inequality applies Hoeffding's inequality. Combining inequalities \eqref{ineq:truncated hoeffding}, \eqref{ineq:expectation diff}, and \eqref{ineq:truncate diff}, we can assert with probability at least
\begin{equation}
    1-2\exp\bigg(-\frac{Bx^2}{2(k\gamma^{k-1})^{2}}\bigg)-2B\exp\bigg(-\frac{\gamma^{2}}{2\|\wb_{r}^{(t)}\|_{2}^{2}}\bigg)
\end{equation}
that the following inequality holds:
\begin{equation}
    \Bigg|\frac{1}{B}\sum_{(\xb,y)\in S_t}g_{r,j}(\xb,y)-\EE_{(\xb,y)\sim\cD}g_{r,j}(\xb,y)\Bigg|\leq x + k d^{(k-1)/2}\|\wb_{r}^{(t)}\|_{2}^{k-1}\cdot 2\exp\bigg(-\frac{\gamma^{2}}{2\|\wb_{r}^{(t)}\|_{2}^{2}}\bigg).
\end{equation}
By setting $\gamma=\sqrt{2\log(16B/\delta)}\|\wb_{r}^{(t)}\|_{2}$ and $x=\sqrt{2}k\gamma^{k-1}\log(8/\delta)/\sqrt{B}$, we establish that with probability at least $1-\delta$, the following bound is true:
\begin{align*}
    &\Bigg|\frac{1}{B}\sum_{(\xb,y)\in S_t}g_{r,j}(\xb,y)-\EE_{(\xb,y)\sim\cD}g_{r,j}(\xb,y)\Bigg|\\
    &\leq\frac{\sqrt{2}k\gamma^{k-1}\log(8/\delta)}{\sqrt{B}}+\frac{k d^{(k-1)/2}\|\wb_{r}^{(t)}\|_{2}^{k-1}\delta}{8B}\\
    &=\frac{2^{k/2}k(\log(16B/\delta))^{(k-1)/2}\log(8/\delta)\|\wb_{r}^{(t)}\|_{2}^{k-1}}{\sqrt{B}}+\frac{k d^{(k-1)/2}\|\wb_{r}^{(t)}\|_{2}^{k-1}\delta}{8B}.
\end{align*}
Applying a union bound over all indices  $r\in[m],j\in[d]$, and iterations $t\in[0,T-1]$, we conclude with probability at least $1-\delta$ that
\begin{align*}
    &\Bigg|\frac{1}{B}\sum_{(\xb,y)\in S_t}g_{r,j}(\xb,y)-\EE_{(\xb,y)\sim\cD}g_{r,j}(\xb,y)\Bigg|\\
    &=\frac{2^{k/2}k(\log(16mdBT/\delta))^{(k-1)/2}\log(8mdT/\delta)\|\wb_{r}^{(t)}\|_{2}^{k-1}}{\sqrt{B}}+\frac{k d^{(k-1)/2}\|\wb_{r}^{(t)}\|_{2}^{k-1}\delta}{8mdBT}\\
    &=\Bigg(\frac{2^{k/2}k(\log(16mdBT/\delta))^{(k-1)/2}\log(8mdT/\delta)}{\sqrt{B}}+\frac{k d^{(k-3)/2}\delta}{8mBT}\Bigg)\cdot\|\wb_{r}^{(t)}\|_{2}^{k-1}.
\end{align*}
\end{proof}
Based on Lemma~\ref{lm:gradient approx error}, we can get the following lemma showing that with high probability, the stochastic sign gradient follows the same update rule as the population sign gradient.
\begin{lemma}\label{lm:sign sgd}
Under Condition \ref{Cond}, with probability at least $1-\delta$ with respect to the randomness of online data selection, the following sign SGD update rule holds:
\begin{align*}
    w_{r,j}^{(t+1)}&=(1-\eta\lambda)w_{r,j}^{(t)}+\eta\cdot\Tilde{\sign}\Bigg(k!a_{r}\cdot\frac{w_{r,1}^{(t)}w_{r,2}^{(t)}\cdots w_{r,k}^{(t)}}{w_{r,j}^{(t)}}\Bigg), &&j\in[k],\\
    w_{r,j}^{(t+1)}&=(1-\eta\lambda)w_{r,j}^{(t)}, &&j\notin[k].
\end{align*}
\end{lemma}
\begin{proof}
We prove this by using induction. We prove the following hypotheses:
\begin{align}
    &\|\wb_{r}^{(t+1)}\|_{2}\leq \|\wb_{r}^{(t)}\|_{2},&&\forall r\in[m].\tag{H$_1$}\label{c2:h1}\\
    &w_{r,j}^{(t+1)}=(1-\eta\lambda)w_{r,j}^{(t)}+\eta\cdot\Tilde{\sign}\Bigg(k!a_{r}\cdot\frac{w_{r,1}^{(t)}w_{r,2}^{(t)}\cdots w_{r,k}^{(t)}}{w_{r,j}^{(t)}}\Bigg),&&\forall r\in[m],j\in[k].\tag{H$_2$}\label{c2:h2}\\
    &w_{r,j}^{(t+1)}=(1-\eta\lambda)w_{r,j}^{(t)}, &&\forall r\in[m],j\notin[k].\tag{H$_3$}\label{c2:h3}\\
    &b_{j}w_{r,j}^{(t)}=1,&&\forall r\in\Omega_{b_1\cdots b_k}\cap\Omega_{\mathrm{g}},\forall j\in[k].\tag{H$_4$}\label{c2:h4}\\
    &b_{j}w_{r,j}^{(t)}=b_{j'}w_{r,j'}^{(t)},&&\forall r\in\Omega_{b_1 b_2\cdots b_k}\cap\Omega_{\mathrm{g}},\forall j,j'\in[k].\tag{H$_5$}\label{c2:h5}\\
    &0<b_{j}w_{r,j}^{(t+1)}\leq(1-\eta\lambda)b_{j}w_{r,j}^{(t)},&&\forall r\in\Omega_{b_1 b_2\cdots b_k}\cap\Omega_{\mathrm{b}},\forall j\in[k].\tag{H$_6$}\label{c2:h6}
\end{align}
We will show that \ref{c2:h2}$(0)$, \ref{c2:h3}$(0)$, \ref{c2:h4}$(0)$ and \ref{c2:h5}$(0)$ are true and for any $t\geq0$ we have
\begin{itemize}[leftmargin=*]
    \item \ref{c2:h2}$(t)$, \ref{c2:h4}$(t)$ $\Longrightarrow$ \ref{c2:h4}$(t+1)$. (This can be established by adapting the proof of Lemma~\ref{lm:exact gd feature}; hence, we omit the proof details here.)
    \item \ref{c2:h2}$(t)$, \ref{c2:h5}$(t)$ $\Longrightarrow$ \ref{c2:h5}$(t+1)$, \ref{c2:h6}$(t)$. (This can be shown by following the proof of Lemma~\ref{lm:exact dud feature}, so the proof details are omitted here.)
    \item \ref{c2:h3}$(t)$, \ref{c2:h4}$(t)$, \ref{c2:h4}$(t+1)$, \ref{c2:h6}$(t)$ $\Longrightarrow$\ref{c2:h1}$(t)$.
    \item $\{$\ref{c2:h1}$(s)\}_{s=0}^{t}$ $\Longrightarrow$ \ref{c2:h2}$(t+1)$, \ref{c2:h3}$(t+1)$.
\end{itemize}
\ref{c2:h4}$(0)$ and \ref{c2:h5}$(0)$ are obviously true since $w_{r,j}^{(0)}=b_j$ for any $r\in\Omega_{b_1 b_2\cdots b_k}$ and $j\in[k]$. To prove that \ref{c2:h2}$(0)$ and \ref{c2:h3}$(0)$ are true, we only need to verify that
\begin{align}
    \Tilde{\sign}\Bigg(\frac{1}{|S_{0}|}\sum_{(\xb,y)\in S_{0}}\sigma'(\la\wb_{r}^{(0)},\xb\ra)\cdot a_r y x_{j}\Bigg)&=\Tilde{\sign}\Bigg(k!a_{r}\cdot\frac{w_{r,1}^{(0)}w_{r,2}^{(0)}\cdots w_{r,k}^{(0)}}{w_{r,j}^{(0)}}\Bigg),&& \forall j\in[k],\label{eq:sign feature}\\
    \Tilde{\sign}\Bigg(\frac{1}{|S_{0}|}\sum_{(\xb,y)\in S_{0}}\sigma'(\la\wb_{r}^{(0)},\xb\ra)\cdot a_r y x_{j}\Bigg)&=0,&& \forall j\notin[k].\label{eq:sign noise}
\end{align}
By Lemma~\ref{lm:gradient approx error}, we have for $j\notin[k]$
\begin{align*}
    \Bigg|\frac{1}{|S_{0}|}\sum_{(\xb,y)\in S_{t}}\sigma'(\la\wb_{r}^{(0)},\xb\ra)\cdot a_r y x_{j}\Bigg|&\leq\epsilon_1\cdot\|\wb_{r}^{(0)}\|_{2}^{k-1}=\epsilon_1\cdot d^{\frac{k-1}{2}}<\rho,
\end{align*}
leading to \eqref{eq:sign noise}. For $j\in[k]$, we have
\begin{align*}
    \Bigg|\frac{1}{|S_{0}|}\sum_{(\xb,y)\in S_{0}}\sigma'(\la\wb_{r}^{(0)},\xb\ra)\cdot a_r y x_{j}-k!a_{r}\cdot\frac{w_{r,1}^{(0)}w_{r,2}^{(0)}\cdots w_{r,k}^{(0)}}{w_{r,j}^{(0)}}\Bigg|&\leq\epsilon_1\cdot\|\wb_{r}^{(0)}\|_{2}^{k-1}=\epsilon_1\cdot d^{\frac{k-1}{2}},
\end{align*}
Since $k!-\epsilon_1\cdot d^{\frac{k-1}{2}}\geq\rho$, we can get
\begin{align*}
    \Tilde{\sign}\Bigg(\frac{1}{|S_{0}|}\sum_{(\xb,y)\in S_{0}}\sigma'(\la\wb_{r}^{(0)},\xb\ra)\cdot a_r y x_{j}\Bigg)&=\Tilde{\sign}\Bigg(k!a_{r}\cdot\frac{w_{r,1}^{(0)}w_{r,2}^{(0)}\cdots w_{r,k}^{(0)}}{w_{r,j}^{(0)}}\Bigg),
\end{align*}
which verifies \eqref{eq:sign feature}. Next, we verify that \ref{c2:h3}$(t)$, \ref{c2:h4}$(t)$, \ref{c2:h4}$(t+1)$, \ref{c2:h6}$(t)$ $\Longrightarrow$\ref{c2:h1}$(t)$. For $r\in\Omega_{\mathrm{g}}$, given \ref{c2:h3}$(t)$, \ref{c2:h4}$(t)$ and \ref{c2:h4}$(t+1)$, we can get
\begin{align*}
    \|\wb_{r}^{(t+1)}\|_{2}&=\Bigg(\sum_{j=1}^{d}\big(w_{r,j}^{(t+1)}\big)^{2}\Bigg)^{\frac{1}{2}}=\Bigg(\sum_{j=1}^{k}\big(w_{r,j}^{(t)}\big)^{2}+\sum_{j=k+1}^{d}\big((1-\eta\lambda)w_{r,j}^{(t)}\big)^{2}\Bigg)^{\frac{1}{2}}\leq\|\wb_{r}^{(t)}\|_{2}.
\end{align*}
For $r\in\Omega_{\mathrm{b}}$, given \ref{c2:h3}$(t)$ and \ref{c2:h6}$(t)$, we can get
\begin{align*}
    \|\wb_{r}^{(t+1)}\|_{2}&=\Bigg(\sum_{j=1}^{d}\big(w_{r,j}^{(t+1)}\big)^{2}\Bigg)^{\frac{1}{2}}\leq\Bigg(\sum_{j=1}^{d}\big((1-\eta\lambda)w_{r,j}^{(t)}\big)^{2}\Bigg)^{\frac{1}{2}}\leq\|\wb_{r}^{(t)}\|_{2}.
\end{align*}
Finally, we verify that $\{$\ref{c2:h1}$(s)\}_{s=0}^{t}$ $\Longrightarrow$ \ref{c2:h2}$(t+1)$, \ref{c2:h3}$(t+1)$. Notice that $\|\wb_{r}^{(t+1)}\|_{2}\leq\|\wb_{r}^{(0)}\|_{2}$ given $\{$\ref{c2:h1}$(s)\}_{s=0}^{t}$, we can prove \ref{c2:h3}$(t+1)$ and \ref{c2:h2}$(t+1)$ by following the prove of \eqref{eq:sign feature} and \eqref{eq:sign noise} given Lemma~\ref{lm:gradient approx error}.
\end{proof}
Based on Lemma~\ref{lm:sign sgd} and the proof of Lemma~\ref{lm:exact order}, Lemma~\ref{lm:nn approx}, we can get the following lemmas and theorems aligning with the result of population sign GD.
\begin{lemma}\label{lm:exact order2}
Under Condition~\ref{Cond}, for $T=\Theta(k\eta^{-1}\lambda^{-1}\log(d))$,with a probability of at least $1-\delta$ with respect to the randomness of the online data selection, it holds that
\begin{align*}
    w_{r,j}^{(T)}&=b_j,&&\forall r\in\Omega_{b_1b_1\cdots b_k}\cap\Omega_{\mathrm{g}},j\in[k],\\
    |w_{r,j}^{(T)}|&\leq d^{-(k+1)}, &&\forall r\in\Omega_{\mathrm{g}},j\in[d]\setminus[k],\\
    |w_{r,j}^{(T)}|&\leq d^{-(k+1)}, &&\forall r\in\Omega_{\mathrm{b}},j\in[d].
\end{align*}
\end{lemma}

\begin{lemma}\label{lm:nn approx2}
Under Condition~\ref{Cond}, with a probability of at least $1-2\delta$ with respect to the randomness in the neural network's initialization and the online data selection, trained neural network $f(\Wb^{(T)},\xb)$ approximates accurate classifier $(m/2^{k+1})\cdot f(\Wb^{*},\xb)$ well: 
\begin{equation*}
    \PP_{\xb\sim\cD_{\xb}}\bigg(\frac{f(\Wb^{(T)},\xb)}{(m/2^{k+1})\cdot f(\Wb^{*},\xb)}\in[0.5,1.5]\bigg)\geq 1-\epsilon.
\end{equation*}
\end{lemma}

Based on Lemma~\ref{lm:nn approx2}, we are now ready to prove our main theorem.
\begin{theorem}\label{thm:test error2}
Under Condition~\ref{Cond}, we run mini-batch SGD for $T=\Theta(k\eta^{-1}\lambda^{-1}\log(d))$ iterations. Then with probability at least $1-2\delta$ with respect to the randomness of neural network initialization and the online data selection, it holds that
\begin{align*}
\mathbb{P}_{(\xb,y)\sim\cD}(yf(\Wb^{(T)}, \xb) \geq \gamma m) \geq 1-\epsilon.    
\end{align*}
where $\gamma=0.25k!$ is a constant. 
\end{theorem}
\begin{proof}
Given Lemma~\ref{lm:nn approx2}, we have
\begin{align*}
    &\PP_{\xb\sim\cD_{\xb}}\bigg(\frac{f(\Wb^{(T)},\xb)}{(m/2^{k+1})\cdot f(\Wb^{*},\xb)}\geq0.5\bigg)\\
    &\geq\PP_{\xb\sim\cD_{\xb}}\bigg(\frac{f(\Wb^{(T)},\xb)}{(m/2^{k+1})\cdot f(\Wb^{*},\xb)}\in[0.5,1.5]\bigg)\\
    &\geq 1-\epsilon.
\end{align*}
According to Proposition~\ref{prop:symmetry}, we can get
\begin{align*}
    &\PP_{\xb\sim\cD_{\xb}}\bigg(\frac{f(\Wb^{(T)},\xb)}{(m/2^{k+1})\cdot f(\Wb^{*},\xb)}\geq0.5\bigg)\\
    &=\PP_{\xb\sim\cD_{\xb}}\bigg(\frac{yf(\Wb^{(T)},\xb)}{(m/2^{k+1})\cdot k!\cdot 2^{k}}\geq0.5\bigg)\\
    &=\PP_{\xb\sim\cD_{\xb}}\big(yf(\Wb^{(T)},\xb)\geq 0.25k!\cdot m\big),
\end{align*}
which completes the proof.
\end{proof}
\section{Trainable Second Layer}\label{sec:D}
In this section, we consider sign SGD for training the first and second layers together. In this scenario, we have the following sign SGD update rule:
\begin{align}
    \wb_{r}^{(t+1)}&=(1-\lambda\eta)\wb_{ r}^{(t)}+\eta\cdot\Tilde{\sign}\bigg(\frac{1}{|S_t|}\sum_{(\xb,y)\in S_{t}}\sigma'(\la\wb_{r}^{(t)},\xb\ra)\cdot a_r^{(t)} y\xb\bigg),\label{eq:w}\\
    a_{r}^{(t+1)}&=a_{ r}^{(t)}+\eta_{2}\cdot\Tilde{\sign}\bigg(\frac{1}{|S_t|}\sum_{(\xb,y)\in S_{t}}\sigma(\la\wb_{r}^{(t)},\xb\ra)\bigg),\label{eq:a}
\end{align}
where $|S_t|=B$. For training the neural network over $T=\Theta(k(\eta\lambda)^{-1}\log(d))$ iterations, we adopt a small learning rate for the second layer, adhering to the condition:
\begin{align*}
    \eta_2\leq\frac{1}{4T}\sqrt{\frac{\pi k}{8}}\Big(\frac{e+e^{-1}}{2}\Big)^{-k}.
\end{align*}

The network is initialized symmetrically: for every $r\in[m]$, initialize $\wb_{r}^{(0)}\sim\text{Unif}(\{-1,1\}^{d})$ and initialize $a_{r}^{(0)}\sim\text{Unif}(\{-1,1\})$. Under this setting, denote
\begin{align*}
    \Omega_{\mathrm{g}}=\bigg\{r\in[m]\,\bigg|\,a_r=\prod_{j=1}^{k}\sign(w_{r,j}^{(0)})\bigg\}, \Omega_{\mathrm{b}}=\bigg\{r\in[m]\,\bigg|\,a_r=-\prod_{j=1}^{k}\sign(w_{r,j}^{(0)})\bigg\}.
\end{align*}

Similar to the fix-second-layer case, our initial step involves estimating the approximation error between the SGD gradient and the population gradient, detailed within the lemma that follows.

\begin{lemma}\label{lm:gradient approx error2}
With probability at least $1-\delta$ with respect to the randomness of online data selection, for all $t \leq T$, we have for any $r\in[m]$ and $j\in[d]$ that 
\begin{align}
    \Bigg|\frac{1}{|S_t|}\sum_{(\xb,y)\in S_{t}}\sigma'(\la\wb_{r}^{(t)},\xb\ra)\cdot a_r^{(t)} yx_{j}-\EE\big[\sigma'(\la\wb_{r}^{(t)},\xb\ra)\cdot a_{r}^{(t)}yx_{j}\big]\Bigg|&\leq\epsilon_{1}\cdot|a_{r}^{(t)}|\|\wb_r^{(t)}\|_{2}^{k-1},\label{ineq:approx3}
\end{align}
where
\begin{align*}
    \epsilon_1&=\frac{2^{k/2}k(\log(16mdBT/\delta))^{(k-1)/2}\log(8mdT/\delta)}{\sqrt{B}}+\frac{k d^{(k-3)/2}\delta}{8mBT}.
\end{align*}
\end{lemma}
\begin{proof}
To prove \eqref{ineq:approx3}, we denote
\begin{align*}
    g_{r,j}(\xb,y)&=\sigma'(\la\wb_{r}^{(t)},\xb\ra)\cdot a_{r}^{(t)}yx_{j}=k(\la\wb_{r}^{(t)},\xb\ra)^{k-1}\cdot a_{r}^{(t)}yx_{j},\\
    h_{r,j}(\xb,y)&=g_{r,j}(\xb,y)\cdot\ind\big[|\la\wb_{r}^{(t)},\xb\ra|\leq\gamma\big].
\end{align*}
Initially, by invoking Hoeffding's inequality, we have the following probability bound:
\begin{equation}\label{ineq:truncated hoeffding1}
    \PP\Bigg(\Bigg|\frac{1}{B}\sum_{(\xb,y)\in S_t}h_{r,j}(\xb,y)-\EE_{(\xb,y)\sim\cD}h_{r,j}(\xb,y)\Bigg|\geq x\Bigg)\leq2\exp\Bigg(-\frac{Bx^2}{2(k \gamma^{k-1}a_{r}^{(t)})^{2}}\Bigg).
\end{equation}
Next, we establish an upper bound for the difference between the expected values of $h_{r,j}(\xb,y)$ and $g_{r,j}(\xb,y)$:
\begin{equation}\label{ineq:expectation diff1}
\begin{aligned}
    \Big|\EE_{(\xb,y)\sim\cD}h_{r,j}(\xb,y)-\EE_{(\xb,y)\sim\cD}g_{r,j}(\xb,y)\Big|&=\Big|\EE_{(\xb,y)\sim\cD}g_{r,j}(\xb,y)\cdot\ind\big[|\la\wb_{r}^{(t)},\xb\ra|>\gamma\big]\Big|\\
    &\leq k d^{(k-1)/2}\|\wb_{r}^{(t)}\|_{2}^{k-1}|a_{r}^{(t)}|\cdot\PP\big(|\la\wb_{r}^{(t)},\xb\ra|>\gamma\big)\\
    &\leq k d^{(k-1)/2}\|\wb_{r}^{(t)}\|_{2}^{k-1}|a_{r}^{(t)}|\cdot 2\exp\bigg(-\frac{\gamma^{2}}{2\|\wb_{r}^{(t)}\|_{2}^{2}}\bigg),
\end{aligned}
\end{equation}
where the first inequality follows from the Cauchy-Schwarz inequality and the second from Hoeffding's inequality. With high probability, the gradient and the truncated gradient coincide:
\begin{equation}\label{ineq:truncate diff1}
\begin{aligned}
    &\PP\Bigg(\Bigg|\frac{1}{B}\sum_{(\xb,y)\in S_t}h_{r,j}(\xb,y)-\frac{1}{B}\sum_{(\xb,y)\in S_t}g_{r,j}(\xb,y)\Bigg|=0\Bigg)\\
    &=\PP\Bigg(\Bigg|\frac{1}{B}\sum_{(\xb,y)\in S_t}g_{r,j}(\xb,y)\ind\big[|\la\wb_{r}^{(t)},\xb\ra|>\gamma\big]\Bigg|=0\Bigg)\\
    &\geq\PP\Big(|\la\wb_{r}^{(t)},\xb\ra|\leq\gamma,\forall (\xb,y)\in S_t\Big)\\
    &\geq\Bigg(1-2\exp\bigg(-\frac{\gamma^{2}}{2\|\wb_{r}^{(t)}\|_{2}^{2}}\bigg)\Bigg)^{B}\\
    &\geq 1-2B\exp\bigg(-\frac{\gamma^{2}}{2\|\wb_{r}^{(t)}\|_{2}^{2}}\bigg).
\end{aligned}
\end{equation}
Combing \eqref{ineq:truncated hoeffding1}, \eqref{ineq:expectation diff1} and \eqref{ineq:truncate diff1}, it holds with probability at least
\begin{equation*}
    1-2\exp\Big(-\frac{Bx^2}{2(k\gamma^{k-1}a_{r}^{(t)})^{2}}\Big)-2B\exp\bigg(-\frac{\gamma^{2}}{2\|\wb_{r}^{(t)}\|_{2}^{2}}\bigg)
\end{equation*}
that
\begin{equation*}
    \Bigg|\frac{1}{B}\sum_{(\xb,y)\in S_t}g_{r,j}(\xb,y)-\EE_{(\xb,y)\sim\cD}g_{r,j}(\xb,y)\Bigg|\leq x + k d^{(k-1)/2}\|\wb_{r}^{(t)}\|_{2}^{k-1}|a_{r}^{(t)}|\cdot 2\exp\bigg(-\frac{\gamma^{2}}{2\|\wb_{r}^{(t)}\|_{2}^{2}}\bigg).
\end{equation*}
By taking $\gamma=\sqrt{2\log(16B/\delta)}\|\wb_{r}^{(t)}\|_{2}$ and $x=\sqrt{2}k\gamma^{k-1}|a_{r}^{(t)}|\log(8/\delta)/\sqrt{B}$, then with probability at least $1-\delta$ it holds that
\begin{align*}
    &\Bigg|\frac{1}{B}\sum_{(\xb,y)\in S_t}g_{r,j}(\xb,y)-\EE_{(\xb,y)\sim\cD}g_{r,j}(\xb,y)\Bigg|\\
    &\leq\frac{\sqrt{2}k\gamma^{k-1}|a_{r}^{(t)}|\log(8/\delta)}{\sqrt{B}}+\frac{k d^{(k-1)/2}\|\wb_{r}^{(t)}\|_{2}^{k-1}|a_{r}^{(t)}|\delta}{8B}\\
    &=\frac{2^{k/2}k(\log(16B/\delta))^{(k-1)/2}\log(8/\delta)|a_{r}^{(t)}|\|\wb_{r}^{(t)}\|_{2}^{k-1}}{\sqrt{B}}+\frac{k d^{(k-1)/2}|a_{r}^{(t)}|\|\wb_{r}^{(t)}\|_{2}^{k-1}\delta}{8B}.
\end{align*}
Then, by applying a union bound to all $r\in[m],j\in[d]$ and iterations $t\in[0,T-1]$, it holds with probability at least $1-\delta$ that
\begin{align*}
    &\Bigg|\frac{1}{B}\sum_{(\xb,y)\in S_t}g_{r,j}(\xb,y)-\EE_{(\xb,y)\sim\cD}g_{r,j}(\xb,y)\Bigg|\\
    &=\frac{2^{k/2}k(\log(16mdBT/\delta))^{(k-1)/2}\log(8mdT/\delta)\|\wb_{r}^{(t)}\|_{2}^{k-1}|a_{r}^{(t)}|}{\sqrt{B}}+\frac{k d^{(k-1)/2}\|\wb_{r}^{(t)}\|_{2}^{k-1}|a_{r}^{(t)}|\delta}{8mdBT}\\
    &=\Bigg(\frac{2^{k/2}k(\log(16mdBT/\delta))^{(k-1)/2}\log(8mdT/\delta)}{\sqrt{B}}+\frac{k d^{(k-3)/2}\delta}{8mBT}\Bigg)\cdot|a_{r}^{(t)}|\|\wb_{r}^{(t)}\|_{2}^{k-1}.
\end{align*}
\end{proof}

\begin{lemma}[Stability of Second Layer Weights]\label{lm:second-layer}
For $t\leq T=\Theta(k(\eta\lambda)^{-1}\log(d))$, the magnitude of change in the second layer weights from their initial values is bounded as follows:
\begin{align*}
    |a_{r}^{(t)}-a_{r}^{(0)}|&\leq c,
\end{align*}
where
\begin{align*}
    c=\frac{1}{4}\sqrt{\frac{\pi k}{8}}\Big(\frac{e+e^{-1}}{2}\Big)^{-k}.
\end{align*}
Consequently, the sign of each weight remains consistent over time:
\begin{align*}
    \sign(a_{r}^{(t)})=\sign(a_{r}^{(0)}).
\end{align*}
\end{lemma}
\begin{proof}
Notice that by \eqref{eq:a} and $\Tilde{\sign}(\cdot)\in\{-1,0,1\}$, we can get for any $t\geq0$ that
\begin{align*}
    a_{ r}^{(t)}-\eta_2\leq a_{r}^{(t+1)}\leq a_{ r}^{(t)}+\eta_2,
\end{align*}
which implies that
\begin{align*}
    |a_{r}^{(t)}-a_{r}^{(0)}|&\leq \eta_{2}t\leq \eta_{2}T\leq c,
\end{align*}
where the second inequality is by $t \leq T$, and the last inequality is by the condition $\eta_2\leq\frac{1}{4T}\sqrt{\frac{\pi k}{8}}\Big(\frac{e+e^{-1}}{2}\Big)^{-k}$.
\end{proof}

\begin{lemma}\label{lm:sign sgd2}
With probability at least $1-\delta$ with respect to the randomness of online data selection, the following sign SGD update rule holds:
\begin{align*}
    w_{r,j}^{(t+1)}&=(1-\eta\lambda)w_{r,j}^{(t)}+\eta\cdot\Tilde{\sign}\Bigg(k!a_{r}^{(t)}\cdot\frac{w_{r,1}^{(t)}w_{r,2}^{(t)}\cdots w_{r,k}^{(t)}}{w_{r,j}^{(t)}}\Bigg), &&j\in[k],\\
    w_{r,j}^{(t+1)}&=(1-\eta\lambda)w_{r,j}^{(t)}, &&j\notin[k].
\end{align*}
\end{lemma}
\begin{proof}
We prove this by using induction. We prove the following hypotheses:
\begin{align}
    &\|\wb_{r}^{(t+1)}\|_{2}\leq\|\wb_{r}^{(t)}\|_{2}, &&\forall r\in[m].\tag{H$_1$}\label{d3:h1}\\
    &w_{r,j}^{(t+1)}=(1-\eta\lambda)w_{r,j}^{(t)}+\eta_{1}\cdot\Tilde{\sign}\Bigg(k!a_{r}^{(t)}\cdot\frac{w_{r,1}^{(t)}w_{r,2}^{(t)}\cdots w_{r,k}^{(t)}}{w_{r,j}^{(t)}}\Bigg), &&\forall r\in[m],j\in[k].\tag{H$_2$}\label{d3:h2}\\
    &w_{r,j}^{(t+1)}=(1-\eta\lambda)w_{r,j}^{(t)},&&\forall r\in[m],j\notin[k].\tag{H$_3$}\label{d3:h3}\\
    &b_{j}w_{r,j}^{(t)}=1,&&\forall r\in\Omega_{b_1\cdots b_k}\cap\Omega_{\mathrm{g}},\forall j\in[k].\tag{H$_4$}\label{d3:h4}\\
    &b_{j}w_{r,j}^{(t)}=b_{j'}w_{r,j'}^{(t)},&&\forall r\in\Omega_{b_1 b_2\cdots b_k}\cap\Omega_{\mathrm{g}},\forall j,j'\in[k].\tag{H$_5$}\label{d3:h5}\\
    &0<b_{j}w_{r,j}^{(t+1)}\leq(1-\eta\lambda)b_{j}w_{r,j}^{(t)},&&\forall r\in\Omega_{b_1 b_2\cdots b_k}\cap\Omega_{\mathrm{b}},\forall j\in[k].\tag{H$_6$}\label{d3:h6}
\end{align}
We will show that \ref{d3:h2}$(0)$, \ref{d3:h3}$(0)$, \ref{d3:h4}$(0)$ and \ref{d3:h5}$(0)$ are true and for any $t\geq0$ we have
\begin{itemize}[leftmargin=*]
    \item \ref{d3:h2}$(t)$, \ref{d3:h4}$(t)$ $\Longrightarrow$ \ref{d3:h4}$(t+1)$. 
    \item \ref{d3:h2}$(t)$, \ref{d3:h5}$(t)$ $\Longrightarrow$ \ref{d3:h5}$(t+1)$, \ref{d3:h6}$(t)$. 
    \item \ref{d3:h3}$(t)$, \ref{d3:h4}$(t)$, \ref{d3:h4}$(t+1)$, \ref{d3:h6}$(t)$ $\Longrightarrow$\ref{d3:h1}$(t)$.
    \item $\{$\ref{d3:h1}$(s)\}_{s=0}^{t}$ $\Longrightarrow$ \ref{d3:h2}$(t+1)$, \ref{d3:h3}$(t+1)$.
\end{itemize}
\ref{d3:h4}$(0)$ and \ref{d3:h5}$(0)$ are obviously true since $w_{r,j}^{(0)}=b_j$ for any $r\in\Omega_{b_1 b_2\cdots b_k}$ and $j\in[k]$. To prove that \ref{d3:h2}$(0)$ and \ref{d3:h3}$(0)$ are true, we can follow the proof of Lemma~\ref{lm:sign sgd} by noticing that $|a_{r}^{(0)}|=1$. Now, we verify that \ref{d3:h2}$(t)$, \ref{d3:h4}$(t)$ $\Longrightarrow$ \ref{d3:h4}$(t+1)$. By \ref{d3:h2}$(t)$ and $\sign(a_{r}^{(t)})=\sign(a_{r}^{(0)})$ according to Lemma~\ref{lm:second-layer}, we have for any neuron $r\in\Omega_{b_1 b_2\cdots b_k}\cap\Omega_{\mathrm{g}}$ that
\begin{align*}
    b_j w_{r,j}^{(t+1)}&=(1-\eta\lambda)b_j w_{r,j}^{(t)}\\
    &\qquad+\eta\cdot\frac{\sign(b_1 w_{r,1}^{(t)})\sign(b_2 w_{r,2}^{(t)})\cdots\sign(b_k w_{r,k}^{(t)})}{\sign(b_j w_{r,j}^{(t)})}\cdot\ind\Bigg[\frac{|a_{r}^{(0)}w_{r,1}^{(t)}w_{r,2}^{(t)}\cdots w_{r,k}^{(t)}|}{|w_{r,j}^{(t)}|}\geq\frac{\rho}{k!}\Bigg]\\
    &=(1-\eta\lambda)+\eta\cdot\ind[k!\geq\rho]\\
    &=1,
\end{align*}
where the last equality is by $\rho\leq k!(1-c)\leq k!\cdot|a_{r}^{(t)}|$. \ref{d3:h2}$(t)$, \ref{d3:h5}$(t)$ $\Longrightarrow$ \ref{d3:h5}$(t+1)$, \ref{d3:h6}$(t)$ can be verified in the same way as Lemma~\ref{lm:exact dud feature} by noticing that $\rho\leq k!(1-c)\leq k!\cdot|a_{r}^{(t)}|$. \ref{d3:h3}$(t)$, \ref{d3:h4}$(t)$, \ref{d3:h4}$(t+1)$, \ref{d3:h6}$(t)$ $\Longrightarrow$\ref{d3:h1}$(t)$ can be proved by following exactly the same proof as Lemma~\ref{lm:exact dud feature}. $\{$\ref{d3:h1}$(s)\}_{s=0}^{t}$ $\Longrightarrow$ \ref{d3:h2}$(t+1)$, \ref{d3:h3}$(t+1)$ be verified in the same way as Lemma~\ref{lm:exact dud feature} by noticing that $\epsilon_1\cdot(1+c)\cdot d^{\frac{k-1}{2}}<\rho$ and $k!-\epsilon_1\cdot(1+c)\cdot d^{\frac{k-1}{2}}>\rho$. 
\end{proof}
Based on Lemma~\ref{lm:second-layer} and Lemma~\ref{lm:sign sgd2}, we can get the following lemmas and theorems aligning with the result of the fixed second-layer case.
\begin{lemma}\label{lm:exact order3}
For $T=\Theta(k\eta^{-1}\lambda^{-1}\log(d))$,with a probability of at least $1-\delta$ with respect to the randomness of the online data selection, it holds that
\begin{align*}
    w_{r,j}^{(T)}&=b_j,&&\forall r\in\Omega_{b_1b_1\cdots b_k}\cap\Omega_{\mathrm{g}},j\in[k],\\
    |w_{r,j}^{(T)}|&\leq d^{-(k+1)}, &&\forall r\in\Omega_{\mathrm{g}},j\in[d]\setminus[k],\\
    |w_{r,j}^{(T)}|&\leq d^{-(k+1)}, &&\forall r\in\Omega_{\mathrm{b}},j\in[d].
\end{align*}
\end{lemma}

\begin{lemma}
With a probability of at least $1-2\delta$ with respect to the randomness in the neural network's initialization and the online data selection, trained neural network $f(\Wb^{(T)},\xb)$ approximates accurate classifier $(m/2^{k+1})\cdot f(\Wb^{*},\xb)$ well: 
\begin{equation*}
    \PP_{\xb\sim\cD_{\xb}}\bigg(\frac{f(\Wb^{(T)},\xb)}{(m/2^{k+1})\cdot f(\Wb^{*},\xb)}\in[0.25,1.75]\bigg)\geq1-\epsilon.
\end{equation*}
\end{lemma}
\begin{proof}
Let
\begin{align*}
    \tilde{f}(\Wb^{(T)},\xb)=\sum_{r=1}^{m}a_{r}^{(0)}\cdot\sigma(\la\wb_{r}^{(T)},\xb\ra).
\end{align*}
By the proof of Lemma~\ref{lm:nn approx}, we can get
\begin{align*}
    \frac{|\tilde{f}(\Wb^{(T)},\xb)-(m/2^{k+1})\cdot f(\Wb^{*},\xb)|}{(m/2^{k+1})\cdot |f(\Wb^{*},\xb)|}\leq 0.5.
\end{align*}
To prove the result, we need to estimate the difference between $\tilde{f}(\Wb^{(T)},\xb)$ and $f(\Wb^{(T)},\xb)$:
\begin{equation}\label{ineq:f diff}
\begin{aligned}
    \big|f(\Wb^{(T)},\xb)-\tilde{f}(\Wb^{(T)},\xb)\big|&=\Bigg|\sum_{r=1}^{m}(a_{r}^{(T)}-a_{r}^{(0)})\cdot\sigma(\la\wb_{r}^{(T)},\xb\ra)\Bigg|\\
    &\leq\sum_{r=1}^{m}|a_{r}^{(T)}-a_{r}^{(0)}|\cdot|\la\wb_{r}^{(T)},\xb\ra|^{k}\\
    &\leq c\sum_{r=1}^{m}|\la\wb_{r}^{(T)},\xb\ra|^{k}\\
    &=c\underbrace{\sum_{r\in\Omega_{\mathrm{g}}}|\la\wb_{r}^{(T)},\xb\ra|^{k}}_{I_1}+c\underbrace{\sum_{r\in\Omega_{\mathrm{b}}}|\la\wb_{r}^{(T)},\xb\ra|^{k}}_{I_2}
\end{aligned}
\end{equation}
where the first inequality is by triangle inequality; the second inequality is by Lemma~\ref{lm:second-layer}. Then, we provide upper bounds for terms $I_1$ and $I_2$ respectively. For $I_1$, we have the following upper bound:
\begin{align}
    I_1&=\sum_{r\in\Omega_{\mathrm{g}}}|\la\wb_{r,[1:k]}^{(T)},\xb_{[1:k]}\ra+\la\wb_{r,[k+1:d]}^{(T)},\xb_{[k+1:d]}\ra|^{k}\notag\\
    &\leq\sum_{r\in\Omega_{\mathrm{g}}}|\la\wb_{r,[1:k]}^{(T)},\xb_{[1:k]}\ra|^{k}+\sum_{r\in\Omega_{\mathrm{g}}}|\la\wb_{r,[k+1:d]}^{(T)},\xb_{[k+1:d]}\ra|\cdot k\big(|\la\wb_{r,[1:k]}^{(T)},\xb_{[1:k]}\ra|+|\la\wb_{r,[k+1:d]}^{(T)},\xb_{[k+1:d]}\ra|\big)^{k}\notag\\
    &\leq\sum_{(b_1,\cdots,b_k)\in\{\pm1\}^{k}}\sum_{r\in\Omega_{b_1 b_2\cdots b_k}\cap\Omega_{\mathrm{g}}}\bigg|\sum_{j=1}^{k}b_{j}x_{j}\bigg|^{k}+\sum_{r\in\Omega_{\mathrm{g}}}d^{-k}\cdot k\big(k+d^{-k}\big)^{k}\notag\\
    &\leq\sum_{(b_1,\cdots,b_k)\in\{\pm1\}^{k}}\Big(\frac{1+\alpha}{2^{k+1}}\Big)m\cdot\bigg|\sum_{j=1}^{k}b_{j}x_{j}\bigg|^{k}+\Big(\frac{1+\alpha}{2}\Big)m\cdot d^{-k}\cdot k\big(k+d^{-k}\big)^{k}\notag\\
    &\leq\Big(\frac{1+\alpha}{2^{k+1}}\Big)m\cdot2k^{k}(1+e^{-2})^{k}+\Big(\frac{1+\alpha}{2}\Big)m\cdot d^{-k}\cdot k\big(k+d^{-k}\big)^{k},\label{ineq:I1 upbound}
\end{align}
where the first inequality is by mean value theorem; the second inequality is by \eqref{ineq:noise part} and Lemma~\ref{lm:exact order3}; the third inequality is by Lemma~\ref{lm:set size}; the last inequality is by Lemma~\ref{lm:auxi4}. For $I_2$, we have the following upper bound
\begin{align}
    I_2&\leq|\Omega_{\mathrm{b}}|\cdot(2d^{-k})^{k}\leq\Big(\frac{1+\alpha}{2}\Big)m\cdot(2d^{-k})^{k}.\label{ineq:I2 upbound}
\end{align}
By plugging \eqref{ineq:I1 upbound} and \eqref{ineq:I2 upbound} into \eqref{ineq:f diff}, we can get:
\begin{align*}
    &\big|f(\Wb^{(T)},\xb)-\tilde{f}(\Wb^{(T)},\xb)\big|\\
    &\leq c\big((1+\alpha)2^{-k}m\cdot k^{k}(1+e^{-2})^{k}+0.5(1+\alpha)m\cdot d^{-k}\cdot k\big(k+d^{-k}\big)^{k}+0.5(1+\alpha)m\cdot(2d^{-k})^{k}\big).
\end{align*}
Therefore, we have
\begin{align*}
    &\frac{\big|f(\Wb^{(T)},\xb)-\tilde{f}(\Wb^{(T)},\xb)\big|}{(m/2^{k+1})\cdot |f(\Wb^{*},\xb)|}\\
    &\leq c\cdot\frac{(1+\alpha)2^{-k}m\cdot k^{k}(1+e^{-2})^{k}+0.5(1+\alpha)m\cdot d^{-k}\cdot k\big(k+d^{-k}\big)^{k}+0.5(1+\alpha)m\cdot(2d^{-k})^{k}}{0.5k!\cdot m}\\
    &\leq c\cdot\frac{2(1+\alpha)2^{-k}\cdot k^{k}(1+e^{-2})^{k}+(1+\alpha)\cdot d^{-k}\cdot k\big(k+d^{-k}\big)^{k}+(1+\alpha)\cdot(2d^{-k})^{k}}{\sqrt{2\pi k}(k/e)^{k}}\\
    &=c\cdot\Bigg(\sqrt{\frac{2}{\pi k}}(1+\alpha)\Big(\frac{e+e^{-1}}{2}\Big)^{k}+\sqrt{\frac{2}{\pi k}}\cdot\bigg(1+\frac{d^{-k}}{k}\bigg)^{k-1}\cdot\Big(\frac{e}{d}\Big)^{k}+\sqrt{\frac{2}{\pi k}}\cdot\Big(\frac{2e}{kd^{k}}\Big)^{k}\Bigg)\\
    &\leq c\cdot\sqrt{\frac{8}{\pi k}}\Big(\frac{e+e^{-1}}{2}\Big)^{k}\\
    &=\frac{1}{4},
\end{align*}
where the second inequality is by Stirling's approximation, and the last equality is due to $c=\frac{1}{4}\sqrt{\frac{\pi k}{8}}\big(\frac{e+e^{-1}}{2}\big)^{-k}$. 
\end{proof}

\section{Auxiliary Lemmas}
We first introduce a finite difference operator with step $h$ and order $n$ as follows,
\begin{align*}
\Delta^{n}_{h}[f](x) = \sum_{i=0}^{n}{k\choose i}(-1)^{n-i}f(x+ih).
\end{align*}
The following Lemma calculates the value finite difference operating on the polynomial function.

\begin{lemma}\citep{milne2000calculus}\label{lm:generatorsum2} 
For $f(x) = x^{n}$, we have that $
\Delta^{n}_{h}[f](x) = h^{n}n!$.
\end{lemma}
Based on Lemma~\ref{lm:generatorsum2}, we have the following Lemma, which calculates the margin of $\good$ NNs defined in \eqref{eq:MLP}. 

\begin{lemma}\label{lm:auxi}
For any integer $k$, we have
\begin{align}
\sum_{i=0}^{k}{k\choose i}(-1)^{i}(k-2i)^{k} = 2^{k}k!. \label{eq:generatorsum1}    
\end{align}
\end{lemma}
\begin{proof}
Applying Lemma~\ref{lm:generatorsum2} with $f(x) = x^{k}$, $n = k$ in Lemma \ref{eq:generatorsum1} gives,
\begin{align*}
\sum_{i=0}^{k}{k\choose i}(-1)^{i}(k-2i)^{k} = (-1)^{k}\Delta^{k}_{-2}[f](k) =  (-1)^{k} (-2)^{k}k! = 2^{k}k!,
\end{align*}
where the first equality is due to the definition of finite difference operator and the last equality is due to Lemma~\ref{lm:generatorsum2}.
\end{proof}

\begin{lemma}\label{lm:auxi4}
For any positive integer $k$, it holds that
\begin{equation*}
    \sum_{i=0}^{k} {k\choose i}|k-2i|^{k}\leq2k^{k}(1+e^{-2})^{k}.
\end{equation*}
\end{lemma}
\begin{proof}
We can establish the following inequality:
\begin{align*}
    \sum_{i=0}^{k} {k\choose i}|k-2i|^{k}&\leq2\sum_{i=0}^{\lfloor\frac{k-1}{2}\rfloor} {k\choose i}|k-2i|^{k}\\
    &=2k^{k}\sum_{i=0}^{\lfloor\frac{k-1}{2}\rfloor} {k\choose i}\Big(1-\frac{2i}{k}\Big)^{k}\\
    &\leq2k^{k}\sum_{i=0}^{\lfloor\frac{k-1}{2}\rfloor} {k\choose i}\exp(-2i)\\
    &\leq2k^{k}\sum_{i=0}^{k} {k\choose i}\exp(-2i)\\
    &=2k^{k}(1+e^{-2})^{k},
\end{align*}
where the first inequality is by ${k\choose i}={k\choose k-i}$, the second inequality is by $1-t \leq \exp(-t), \forall t \in \RR$, the last inequality is by ${k\choose i}\exp(-2i) > 0$.
\end{proof}
\newpage

\section*{NeurIPS Paper Checklist}

\begin{enumerate}

\item {\bf Claims}
    \item[] Question: Do the main claims made in the abstract and introduction accurately reflect the paper's contributions and scope?
    \item[] Answer: \answerYes{} 
    \item[] Justification: The claims made in abstract and Section~\ref{sec:intro} are supported by results in Section~\ref{sec:main} and also the additional expeiremnts in the Appendix.  

\item {\bf Limitations}
    \item[] Question: Does the paper discuss the limitations of the work performed by the authors?
    \item[] Answer: \answerYes{} 
    \item[] Justification: The limitation section is presented in page 12. 

\item {\bf Theory Assumptions and Proofs}
    \item[] Question: For each theoretical result, does the paper provide the full set of assumptions and a complete (and correct) proof?
    \item[] Answer: \answerYes{} 
    \item[] Justification: Yes, the full set of assumptions are provided in Section~\ref{sec:problem} and the complete (and correct) proof are provided in the Appendix~\ref{sec:exp}.

    \item {\bf Experimental Result Reproducibility}
    \item[] Question: Does the paper fully disclose all the information needed to reproduce the main experimental results of the paper to the extent that it affects the main claims and/or conclusions of the paper (regardless of whether the code and data are provided or not)?
    \item[] Answer: \answerYes{} 
    \item[] Justification: This paper provides detailed experiment settings in Section~\ref{sec:exp}, which should be sufficient for reproducing the main experimental results. 

\item {\bf Open access to data and code}
    \item[] Question: Does the paper provide open access to the data and code, with sufficient instructions to faithfully reproduce the main experimental results, as described in supplemental material?
    \item[] Answer: \answerNo{} 
    \item[] Justification: The experiments in this paper are based on synthetic data generated to support the theoretical findings. As such, there is no real-world dataset or code to be made openly accessible. The question of open access to data and code is not applicable in this case, as the main contributions are purely theoretical and do not rely on empirical results from specific datasets.

\item {\bf Experimental Setting/Details}
    \item[] Question: Does the paper specify all the training and test details (e.g., data splits, hyperparameters, how they were chosen, type of optimizer, etc.) necessary to understand the results?
    \item[] Answer: \answerYes{} 
    \item[] Justification: This paper provides detailed experiment settings in Section~\ref{sec:exp}, which should be sufficient for reproducing the main experimental results. 

\item {\bf Experiment Statistical Significance}
    \item[] Question: Does the paper report error bars suitably and correctly defined or other appropriate information about the statistical significance of the experiments?
    \item[] Answer: \answerYes{} 
    \item[] Justification: : This paper provides detailed experiment settings in Section~\ref{sec:exp}, which include important details such as the error bars reported in Table 3.

\item {\bf Experiments Compute Resources}
    \item[] Question: For each experiment, does the paper provide sufficient information on the computer resources (type of compute workers, memory, time of execution) needed to reproduce the experiments?
    \item[] Answer: \answerYes{} 
    \item[] Justification: The requirement of Compute Resources are stated in Section~\ref{sec:exp}.
    
\item {\bf Code Of Ethics}
    \item[] Question: Does the research conducted in the paper conform, in every respect, with the NeurIPS Code of Ethics \url{https://neurips.cc/public/EthicsGuidelines}?
    \item[] Answer: \answerYes 
    \item[] Justification: As a theoretical paper, the research presented in this paper has been conducted with the highest ethical standards and is fully compliant with the NeurIPS Code of Ethics. 

\item {\bf Broader Impacts}
    \item[] Question: Does the paper discuss both potential positive societal impacts and negative societal impacts of the work performed?
    \item[] Answer: \answerNA{} 
    \item[] Justification: Our work provides the theoretical understanding of generic optimization algorithm (SGD). Although there might be some potential social impacts on applications, according to the guidelines, we believe our result does not have a direct connection with these issues.
    
\item {\bf Safeguards}
    \item[] Question: Does the paper describe safeguards that have been put in place for responsible release of data or models that have a high risk for misuse (e.g., pretrained language models, image generators, or scraped datasets)?
    \item[] Answer: \answerNA{} 
    \item[] Justification:  As a theoretical paper, this work does not involve the development or release of any models or datasets, particularly those that might be considered high-risk for misuse, such as pretrained language models, image generators, or scraped datasets. Therefore, the question of safeguards for responsible release does not apply.

\item {\bf Licenses for existing assets}
    \item[] Question: Are the creators or original owners of assets (e.g., code, data, models), used in the paper, properly credited and are the license and terms of use explicitly mentioned and properly respected?
    \item[] Answer: \answerNA{} 
    \item[] Justification: This paper is a theoretical work that does not utilize or rely on any existing assets such as code, data, or models from other sources. Therefore, the question of crediting creators or detailing licenses and terms of use for such assets does not apply.

\item {\bf New Assets}
    \item[] Question: Are new assets introduced in the paper well documented and is the documentation provided alongside the assets?
    \item[] Answer: \answerNA{} 
    \item[] Justification: As this is a theoretical paper, it does not introduce or release any new assets such as datasets, code, or models. Thus, there is no need for documentation related to new assets.

\item {\bf Crowdsourcing and Research with Human Subjects}
    \item[] Question: For crowdsourcing experiments and research with human subjects, does the paper include the full text of instructions given to participants and screenshots, if applicable, as well as details about compensation (if any)? 
    \item[] Answer: \answerNA{} 
    \item[] Justification:  This theoretical paper does not conduct any experiments involving crowdsourcing or research with human subjects, thus there are no participants, instructions, or compensation details to report.

\item {\bf Institutional Review Board (IRB) Approvals or Equivalent for Research with Human Subjects}
    \item[] Question: Does the paper describe potential risks incurred by study participants, whether such risks were disclosed to the subjects, and whether Institutional Review Board (IRB) approvals (or an equivalent approval/review based on the requirements of your country or institution) were obtained?
    \item[] Answer: \answerNA{} 
    \item[] Justification:  As this paper is purely theoretical and does not involve any research with human subjects, there are no study participants, thus no associated risks or requirements for Institutional Review Board (IRB) approvals.

\end{enumerate}

\end{document}